\documentclass[sigconf,printacmref=false,nonacm=true]{acmart}

\usepackage{etoolbox}
\newtoggle{arxiv}
\toggletrue{arxiv} % arxiv version, includes all appendices
\renewcommand\footnotetextcopyrightpermission[1]{}

\newcommand{\E}{\mathbb{E}} 
\renewcommand{\P}{\mathbb{P}} 
 
\newcommand{\N}{\mathbb{N}}

\usepackage{mathtools}
\mathtoolsset{showonlyrefs}
\DeclareMathOperator*{\argmax}{arg\,max}
\DeclareMathOperator*{\argmin}{arg\,min}
\DeclarePairedDelimiter{\ceil}{\lceil}{\rceil}
\DeclarePairedDelimiter{\floor}{\lfloor}{\rfloor}

\usepackage{amsthm}
\newtheorem{thm}{Theorem}
\newtheorem{lem}{Lemma}
\newtheorem{clm}{Claim}

\newtheorem{rem}{Remark}

\usepackage[linesnumbered,ruled]{algorithm2e}
\makeatletter
\renewcommand{\@algocf@capt@plain}{above}
\makeatother
\SetKwInput{KwReturn}{Return}

\begin{document}

\title{Robust Multi-Agent Multi-Armed Bandits}

\author{Daniel Vial}
\email{dvial@utexas.edu}
\affiliation{%
  \institution{University of Texas at Austin}
  \country{}
}

\author{Sanjay Shakkottai}
\email{sanjay.shakkottai@utexas.edu}
\affiliation{%
  \institution{University of Texas at Austin}
  \country{}
}

\author{R. Srikant}
\email{rsrikant@illinois.edu}
\affiliation{%
  \institution{University of Illinois at Urbana-Champaign}
  \country{}
}

\begin{abstract}
Recent works have shown that agents facing independent instances of a stochastic $K$-armed bandit can collaborate to decrease regret. However, these works assume that each agent always recommends their individual best-arm estimates to other agents, which is unrealistic in envisioned applications (machine faults in distributed computing or spam in social recommendation systems). Hence, we generalize the setting to include $n$ honest and $m$ malicious agents who recommend best-arm estimates and arbitrary arms, respectively. We first show that even with a single malicious agent, existing collaboration-based algorithms fail to improve regret guarantees over a single-agent baseline. We propose a scheme where honest agents learn who is malicious and dynamically reduce communication with (i.e., ``block'') them. We show that collaboration indeed decreases regret for this algorithm, assuming $m$ is small compared to $K$ but without assumptions on malicious agents' behavior, thus ensuring that our algorithm is robust against any malicious recommendation strategy.
\end{abstract}

\maketitle
\pagestyle{plain}

\section{Introduction} \label{secIntro}

Multi-armed bandits (MABs) are classical models for online learning and decision making. In this paper, we consider a setting where a group of agents collaborates to solve a stochastic MAB. More precisely, agents face separate instances of the same MAB and collaborate -- via limited communication -- to minimize per-agent regret. As motivation, we describe two applications:
\begin{enumerate}
\item Consider a search engine that displays ads alongside search results, and suppose search requests are processed by a large number of machines/agents. In particular, each machine processes a subset of requests and must decide which ad to display (to maximize, for example, the click-through rate). Here the decision problem is naturally modeled as a MAB, with ads as arms. The machines can collaborate by exchanging information (for example, observed click-through rates), but communication is limited by bandwidth.

\item Consider an online recommendation system, e.g., for restaurants. Each user/agent can decide which restaurant to visit based on their past dining experiences, which again can be modeled as a MAB (with restaurants as arms). However, the users can also collaborate by writing and reading reviews. Here two users communicate if one reads the other's review, and communication is limited in the sense that each user likely reads a small fraction of all reviews.
\end{enumerate}

In such applications, it is infeasible (or at least inefficient) for a single agent to explore all arms. Hence, several algorithms have been proposed in which each agent only explores a small subset of \textit{active arms} and occasionally recommends a best-arm estimate to another agent \cite{chawla2020gossiping,sankararaman2019social}. For example, \cite{chawla2020gossiping} has $o(T)$ communication rounds per $T$ arm pulls, and at each round, each agent receives a best-arm estimate from one other agent, adds this estimate to its active set, and discards a poorly-performing active arm. Thus, communication is infrequent, pairwise, and bit-limited, modeling the limitations in our motivating applications. Despite these limitations, \cite{chawla2020gossiping,sankararaman2019social} show that the true best arm eventually spreads to all agents' active sets via recommendations. Combined with the fact that each agent only explores a small number of active arms, per-agent regret is smaller than in the single-agent MAB setting.

However, the regret guarantees from \cite{chawla2020gossiping,sankararaman2019social} require all agents to truthfully report best-arm estimates to other agents, which does not occur in practice. For example, spam reviews can be modeled as bad arm recommendations in the restaurant application, and machines will occasionally fail and stop communicating altogether in the search engine application. In light of these concerns, we study a more realistic setting in which $n$ \textit{honest agents} explore active arm sets and recommend best-arm estimates (similar to \cite{chawla2020gossiping,sankararaman2019social}), but $m$ \textit{malicious agents} recommend arbitrary arms. Similar to \cite{chawla2020gossiping}, we only permit $o(T)$ pairwise arm recommendations per $T$ arm pulls; for simplicity, we also assume the set of all agents (honest and malicious) is connected by a complete graph. 

\subsection{Our contributions}

{\bf Lower bound:} We show the algorithm from \cite{chawla2020gossiping} (the state-of-the-art in the case $m=0$, where it incurs $O( (K/n) \log(T) / \Delta )$ regret), fails in this generalized setting, in the sense that \textit{even a single malicious agent negates the benefit of collaboration}. More precisely, we prove that for any $m \in \N$, honest agents incur $\Omega ( K \log(T) / \Delta )$ regret (where $K$ is the number of arms, $T$ is the horizon, and $\Delta$ is the arm gap), identical in an order sense to a single-agent baseline where each agent plays the MAB in isolation (see Theorem \ref{thmNoBlacklist} and Remark \ref{remNoBlacklistSummary}). This occurs because honest agents using the algorithm from \cite{chawla2020gossiping} trust all recommendations, so malicious agents can recommend, and subsequently force honest agents to explore, all suboptimal arms. In short, the algorithms from \cite{chawla2020gossiping,sankararaman2019social} rely on the idealized assumption that all agents are fully cooperative, and they break down when this assumption fails.

\vspace{4pt} 

\noindent {\bf Blocking algorithm:} Owing to the failure of \cite{chawla2020gossiping} in the generalized setting, we propose an algorithm using a simple idea called \textit{blocking}, roughly defined as follows: if an agent recommends an arm at time $t \in \{1,\ldots,T\}$ and the arm subsequently performs poorly, ignore the agent's recommendations until time $t^2$, i.e., ``block'' the agent. These increasing blocking periods $t^2$ balance two competing forces (see Remarks \ref{remKeyFeature} and \ref{remBlacklistPhen}). First, honest agents who mistakenly recommend bad arms at small $t$ (which occurs due to noise in the rewards for $t = o( \Delta^{-2} )$ \cite{audibert2010best}) are not ignored for too long, so they can later help spread the best arm to other honest agents' active sets. Second, malicious agents who repeatedly recommend bad arms are punished with increasing severity (a bad recommendation at $t$ blocks them until $t^2$, then until $t^4$, then until $t^8$, etc.), which reduces the number of bad arms they can force honest agents to explore. Conceptually, this tradeoff means that in the presence of malicious agents, honest agents face an explore-exploit dilemma when deciding who to communicate with -- in addition to the standard such dilemma when deciding which arm to pull -- which makes learning more difficult than in the cooperative setting (see Remark \ref{remExpExp}).

\vspace{4pt} 

\noindent {\bf Upper bound:} For the proposed algorithm, we show that regret is upper bounded by $O ( ( m + K/n) \log(T) / \Delta )$. Thus, our algorithm is robust against malicious agents, in the sense that \textit{collaboration reduces regret from the single-agent baseline if $m$ is small compared to $K$} (see Theorem \ref{thmOurAlgorithmRegret} and Remark \ref{remOurAlgorithmSummary}). More precisely, the multiplicative constant in our upper bound is $\min \{ m + K/n, K \}$, i.e., for large $m$ we also recover the $O( K \log(T) / \Delta)$ single-agent baseline. This demonstrates that our blocking approach, and in particular its polynomial-length blocking periods, properly balances the aforementioned explore-exploit communication tradeoff. We also note that, somewhat counterintuitively, our algorithm can be robust when $m$ is much larger than $n$ (see Remark \ref{remMvsN}).

\vspace{4pt} 

\noindent {\bf Worst-case robustness:} Our regret upper bound requires \textit{no assumption on the behavior of malicious agents}, i.e., on how they recommend arms (besides a mild measurability condition). Hence, our algoithm is robust against the worst-case behavior of such agents. This is critical because the definition of ``malicious'' is highly domain-dependent (consider the aforementioned spam review and faulty machine applications\footnote{``Malicious" is something of a misnomer when discussing unintentional failures like faulty machines; we use this word to emphasize the worst-case flavor of our approach.}). For example, our algorithm is robust against groups of malicious agents who collude, ``omniscient" malicious agents who observe and exploit the arm pulls and rewards of all honest agents, ``deceitful" malicious agents who initially report good arms to build credibility but later abuse this credibility by reporting bad arms, and any combination thereof.

\vspace{4pt} 

\noindent {\bf Three-regime analysis:} We show that as time progresses, the proposed algorithm passes through three distinct regimes (see Remark \ref{remBlacklistPhen}). We describe them out of order for ease of exposition.
\begin{itemize}
\item {\bf Early regime:} As discussed above, honest agents initially make mistakes, block one another, and prevent the best arm from spreading. Nevertheless, we show that \textit{polynomial-length blocking is mild enough to allow the best arm to spread}. More precisely, we show that after some almost-surely finite time, i.e., one that does not depend on the horizon $T$ (and denoted by $A_{\tau}$ below), the best arm is active and correctly identified by all honest agents.
\item {\bf Late regime:} After time $A_{\tau} \vee T^{1/K}$, honest agents have identified the best arm, so they block any malicious agent who recommends a bad arm. By our blocking schedule, this means a malicious agent can only recommend bad arms at times $T^{1/K}, T^{2/K} , T^{4/K}$, etc.\ -- for a total of  $\log_2 K$ bad arms. Thus, \textit{as soon as time becomes polynomial in $T$, malicious agents are only contacted finitely often as $T \rightarrow \infty$.}
\item {\bf Intermediate regime:} In contrast, between $A_{\tau}$ and $T^{1/K}$, malicious agents can recommend bad arms at $A_{\tau} , A_{\tau}^2 , A_{\tau}^4$, etc.\ -- $\log \log T$ of them as $T \rightarrow \infty$. However, since the best arm is active after $A_{\tau}$, this is roughly equivalent to playing $K$ arms for horizon $T^{1/K}$, which contributes negligible regret $K \log(T^{1/K})/\Delta = \log(T) / \Delta$. Thus, \textit{before time is polynomial in $T$, malicious agents are contacted infinitely often, but the effective horizon is too small to appreciably increase regret.}
\end{itemize}
This analysis is novel compared to the simpler cooperative case, in which active arm sets are eventually fixed and one can treat long-term regret similar to a single-agent MAB (see Remark \ref{remActiveChanges}).

%\vspace{4pt} 

%\noindent {\bf Other analytical contributions:} Our paper is the first in the collaborative multi-agent stochastic MAB literature to treat \textit{non-monotone active arm sets that change infinitely often} (see Remark \ref{remActiveChanges}). Namely, honest agents in our algorithm discard poorly-performing active arms in favor of recommendations (in contrast to \cite{sankararaman2019social}, where active sets grow monotonically), infinitely many times as $T \rightarrow \infty$ (in contrast to \cite{chawla2020gossiping}, where these sets are fixed after some almost-surely finite time).

\begin{rem}
Our approach of increasing blocking periods is similar in spirit to the content moderation policies of several online platforms. For example, Stack Exchange suggests suspensions of 7, 30, and 365 days for successive rule violations \cite{seBlocking}, while Wikipedia blocks users ``longer for persistent violations'' \cite{wikiBlocking}. Thus, our paper provides a formal model and a rigorous analysis of such policies.
\end{rem}

\subsection{Related Work}

Multi-agent MABs with malicious agents were previously studied in \cite{awerbuch2008competitive} (there called \textit{dishonest} agents), but there are two fundamental differences between this work and ours. First, \cite{awerbuch2008competitive} considers non-stochastic/adversarial MABs \cite{auer1995gambling}, in contrast to the stochastic MABs of our work. Second, \cite{awerbuch2008competitive} assumes each agent communicates with all the others between each arm pull, while our algorithm has $o(T)$ pairwise communications per $T$ arm pulls, which models the limited communication in the motivating applications discussed above. We also note multi-agent non-stochastic MABs without malicious agents were studied in \cite{cesa2016delay,kanade2012distributed,seldin2014prediction}.

We are not aware of prior work studying multi-agent stochastic MABs with malicious agents and limited communication (as our paper does). However, papers including \cite{chawla2020multi,sankararaman2019social,buccapatnam2015information,chakraborty2017coordinated,kolla2018collaborative,lalitha2020bayesian,martinez2019decentralized,landgren2016distributed} have studied the fully cooperative case, i.e., the case $m=0$. The aforementioned \cite{chawla2020gossiping,sankararaman2019social} have settings identical to ours, except for our inclusion of malicious agents. We discuss \cite{chawla2020gossiping} in detail in Sections \ref{secGeneralAlgo} and \ref{secExistingAlgo}. \cite{sankararaman2019social} has two shortcomings relative to \cite{chawla2020gossiping}: agents need to know the arm gap $\Delta_2$ and the regret guarantee is weaker than \cite{chawla2020gossiping} when $m=0$. The remaining papers all allow more communication than \cite{chawla2020gossiping}. Namely, \cite{buccapatnam2015information,chakraborty2017coordinated} allow broadcasts instead of pairwise communication, \cite{kolla2018collaborative,lalitha2020bayesian,martinez2019decentralized} allow communication between each arm pull instead of $o(T)$ times per $T$ pulls, and agents in \cite{landgren2016distributed} communicate arm mean estimates instead of indices of estimated best arms (note the former requires more bandwidth per transmission as $T$ grows, while the latter requires $\log K$ bits independent of $T$). In summary, \cite{chawla2020gossiping} features the best regret guarantee and least restrictive assumptions for fully-cooperative multi-agent stochastic MABs. We thus focus on making this particular algorithm robust against malicious agents and use \cite{chawla2020gossiping} as a point of comparison throughout the paper. Nevertheless, we believe our blocking idea can be used to make other algorithms designed for the fully cooperative case more robust against malicious agents.

The larger multi-agent bandits literature includes \cite{hillel2013distributed,szorenyi2013gossip,chawla2020multi,korda2016distributed,shahrampour2017multi}, which all have fundamental differences from our work. Agents in \cite{hillel2013distributed,szorenyi2013gossip} aim to minimize simple instead of cumulative regret. \cite{chawla2020multi,korda2016distributed} consider multi-agent contextual bandits instead of stochastic MABs. Agents in \cite{shahrampour2017multi} face MABs with different reward distributions instead of separate instances of the same MAB.

Finally, we distinguish our setting from two less related lines of work. First, papers including \cite{anandkumar2011distributed,avner2014concurrent,bistritz2018distributed,kalathil2014decentralized,liu2012learning,liu2020competing,mansour2018competing,rosenski2016multi} consider competitive agents, meaning that rewards are smaller if several agents simultaneously pull an arm; in contrast, we assume rewards are independent across honest agents. Second, papers including \cite{gupta2019better,kapoor2019corruption,liu2019data,lykouris2018stochastic} study MABs with adversarial noise, where the agent’s reward observations are corrupted by an adversary. This behavior is different from that of malicious agents in our work, who recommend bad arms but do not alter reward observations.

\subsection{Organization}

The remainder of the paper is organized as follows. Section \ref{secPrelim} discusses preliminaries. In Section \ref{secGeneralAlgo}, we define a general algorithm for multi-agent MABs. Sections \ref{secExistingAlgo} and \ref{secOurAlgo} analyze two cases of this algorithm: the one from \cite{chawla2020gossiping} and the proposed algorithm. In Section \ref{secExperiments}, we provide numerical results. We close in Section \ref{secConclusions}.

\section{Preliminaries} \label{secPrelim}

We consider a stochastic MAB with $K$ arms, denoted $1,\ldots,K$. Arm $k \in \{1,\ldots,K\}$ generates $\text{Bernoulli}(\mu_k)$ rewards for some $\mu_k \in (0,1)$, independent across agents and across successive pulls of the arm.\footnote{We only require the Bernoulli assumption to use the Hoeffding bound, so the results hold when rewards are $[0,1]$-valued and (with minor modification) subgaussian.} We assume the arms are labeled such that $\mu_1 \geq \cdots \geq \mu_K$. We call $1$ the \textit{best arm} and assume it is unique, i.e., $\mu_1 > \mu_2$. For each arm $k$, we let $\Delta_k = \mu_1 - \mu_k \in (0,1)$ denote the $k$-th arm gap, i.e., the difference in means of the best arm and arm $k$.

Our multi-agent system contains $n+m$ total agents ($n,m \in \N$), who are connected by a complete graph and divided into two types. Agents $1, \ldots , n$, called \textit{honest agents}, collaborate (by running a prescribed algorithm) to minimize their individual cumulative regret. More specifically, each $i \in \{1,\ldots,n\}$ faces a separate instance of the MAB defined above and aims to minimize
\begin{equation}
\E R_T^{(i)} =  \sum_{t=1}^T \E ( \mu_1 - \mu_{ I_t^{(i)} } ) = \sum_{t=1}^T \E \Delta_{I_t^{(i)} } ,
\end{equation}
where $T \in \N$ is a time horizon unknown to $i$ and $I_t^{(i)} \in \{1,\ldots,K\}$ is the arm that $i$ pulls at time $t$. Here and moving forward, all random variables are defined on a common probability space $(\Omega,\mathcal{F},\P)$, and expectation is over all randomness (rewards and the forthcoming communication protocol). In contrast to honest agents, agents $n+1,\ldots,n+m$ need not run the prescribed algorithm. We call them \textit{malicious agents} and formally define their behavior in Section \ref{secGeneralAlgo}. Of course, honest agents do not know who is honest and who is malicious; we make no such assumption on malicious agents.

\section{General algorithm} \label{secGeneralAlgo}

We next describe a regret minimization scheme for multi-agent MABs with blocking, defined from the perspective of honest agent $i$ in Algorithm \ref{algGeneral} (we assume all $i \in \{1,\ldots,n\}$ locally execute the algorithm). Time is discrete and indexed by $t$, where (as above) $i$ pulls arm $I_t^{(i)}$ at each $t \in \N$. During certain time slots $A_j \in \N$, hereafter called \textit{communication epochs}, agents communicate. In particular, at time $A_j$, $i$ solicits an arm recommendation from a random agent \textit{not} belonging to a \textit{blocklist} $P_j^{(i)} \subset \{1,\ldots,n+m\}$, i.e., a subset of agents $i$ is unwilling to communicate with. The algorithm from \cite{chawla2020gossiping} is the special case where $P_j^{(i)} = \emptyset\ \forall\ i, j$, i.e., where no blocking occurs (see Section \ref{secExistingAlgo}). In contrast, our algorithm dynamically modifies these blocklists using subroutine \texttt{Update-Blocklists}, in hopes of reducing communication with malicious agents (see Section \ref{secOurAlgo}). In this section, we leave \texttt{Update-Blocklists} unspecified, and we outline Algorithm \ref{algGeneral} as a general approach encompassing both algorithms.

{\bf Initialization:} $i$ begins by initializing communication epochs $A_j = j^{\beta}$, where $\beta > 1$ is an input to the algorithm. Thus, agents communicate $o(T)$ times per $T$ arm pulls, as discussed in the introduction. Moving forward, we call the period between times $A_{j-1}+1$ and $A_j$ (inclusive) the $j$-th \textit{phase}. Line \ref{algInitCommGossip} also initializes the blocklists to empty sets, meaning $i$ is \textit{a priori} willing to communicate with anyone. In Line \ref{algInitActiveSet}, $i$ initializes the current phase $j = 1$ and a subset of arms $S_j^{(i)} = \hat{S}^{(i)} \cup \{ U_j^{(i)}, L_j^{(i)} \}$. Here $\hat{S}^{(i)} \subset \{1,\ldots,K\}$ is an input to the algorithm with size $|\hat{S}^{(i)}| = S$, and $U_j^{(i)}, L_j^{(i)}$ are two arms not belonging to $\hat{S}^{(i)}$. We call $\hat{S}^{(i)}$  \textit{sticky} arms, as $i$ will explore these arms for the duration of the algorithm. In contrast, the arms $U_j^{(i)}$ and $L_j^{(i)}$ will be updated across phases $j$. We define this update shortly; for now, we note $U_j^{(i)}$ and $L_j^{(i)}$ will represent well- and poorly-performing non-sticky arms, respectively. 

{\bf Pulling active arms:} At time $t \in \{ A_{j-1}+1, \ldots , A_j \}$, $i$ pulls the arm $I_t^{(i)} \in S_j^{(i)}$ that maximizes the $\text{UCB}(\alpha)$ index \cite{auer2002finite,bubeck2011pure} (Line \ref{algPullArm}). Here $\alpha > 0$ is an input to the algorithm which trades off exploration and exploitation (in the same manner as the single-agent setting), and $\hat{\mu}_k^{(i)}(t-1)$ and $T_k^{(i)}(t-1)$ are the average reward and number of plays of arm $k$ for agent $i$ before time $t$. We emphasize that $I_t^{(i)} \in S_j^{(i)}$, i.e., $i$ only pulls arms from $S_j^{(i)}$ during phase $j$. Thus, we call $S_j^{(i)}$ the \textit{active set} and its elements \textit{active arms}.

{\bf Updating active arms:} At epoch $A_j$, $i$ records the active arm that it played most frequently in phase $j$ (denoted $B_j^{(i)}$ in Line \ref{algMostPlayed}), calls the aforementioned \texttt{Update-Blocklists} subroutine (Line \ref{algUpdateGossipDist}, left unspecified for this generic algorithm), and solicits an arm recommendation $R_j^{(i)}$ from agent $H_j^{(i)}$ (Line \ref{algGetRec}, to be discussed shortly). If this recommendation is currently active, $i$'s active set remains unchanged for the next phase (Line \ref{algSameActive}). Otherwise, $i$'s new active set contains its sticky set, its best non-sticky arm, and the recommendation. More precisely, $i$ defines the non-sticky arms $U_{j+1}^{(i)}, L_{j+1}^{(i)}$ for the next phase to be the most-played non-sticky from the current phase (Line \ref{algUjPlus1}) and the recommendation (Line \ref{algLjPlus1}), respectively, and $S_{j+1}^{(i)}$ as the union of these arms and the sticky set (Line \ref{algNewActive}). We emphasize that the active set $S_j^{(i)}$ always includes the sticky set $\hat{S}^{(i)}$, but otherwise varies with the phase $j$; the hope is that $1 \in S_j^{(i)}\ \forall\ i$ eventually (i.e., eventually the best arm spreads to all honest agents, who begin enjoying logarithmic regret).

{\bf Arm recommendations:} We model pairwise communication using Algorithm \ref{algGetArm}, which proceeds as follows. A non-blocked agent is chosen uniformly at random (Line \ref{algSampleRecommender} of Algorithm \ref{algGetArm}). If this agent is honest, it recommends its current best-arm estimate (i.e., its most played arm in the current phase); if malicious, it recommends an arbitrary arm (Lines \ref{algRecMostPlayed} and \ref{algRecArbitrary}, respectively). Note Algorithm \ref{algGetArm} is ``black-boxed'', i.e., $i$ provides inputs $i,j, P_j^{(i)}$ and observes outputs $H_j^{(i)}, R_j^{(i)}$, but does not locally execute Algorithm \ref{algGetArm} (which is impossible, since $i$ does not know who is honest and who is malicious). We also note the communication in Algorithm \ref{algGetArm} is where we use the complete graph assumption.

{\bf Malicious agent behavior:} More precisely, if $i$ contacts malicious agent $i'$ at phase $j$, $i$ receives a random arm distributed as $\nu_{j,i}^{(i')}$, where $\nu_{j,i}^{(i')}$ is any $\mathcal{F}$-measurable mapping from $\Omega$ to the set of distributions over $\{1,\ldots,K\}$ (i.e., $\nu_{j,i}^{(i')}$ is a random distribution over arms). Besides this measurability condition (which ensures that expected regret is well-defined), we make no assumptions on malicious agent behavior. Thus, malicious recommendations are essentially arbitrary. Note this permits the case where malicious agents run Algorithm \ref{algGeneral} and recommend best-arm estimates, i.e., where they behave as honest agents. Moving forward, we call $\{ \nu_{j,i}^{(i')} \}_{j \in \N, i \in \{1,\ldots,n\}}$ the \textit{strategy} of malicious agent $i'$, as it defines how $i'$ interacts with all honest agents $i$ at all phases $j$.

\begin{algorithm} \caption{$\texttt{Multi-Agent-MAB-With-Blocking}(\alpha,\beta,\hat{S}^{(i)})$ (executed by each honest agent $i \in \{1,\ldots,n\}$)} \label{algGeneral}

\KwIn{UCB parameter $\alpha > 0$, phase duration parameter $\beta > 1$, sticky set $\hat{S}^{(i)} \subset \{1,\ldots,K\}$ with $|\hat{S}^{(i)}| = S$}

Initialize $A_{j'} = \ceil{ (j')^{\beta} }, P_{j'}^{(i)} = \emptyset\ \forall\ j' \in \N$ \label{algInitCommGossip}

Set $j = 1$, let $U_j^{(i)}, L_j^{(i)}$ be distinct elements of $\{ 1,\ldots,K \} \setminus \hat{S}^{(i)}$, set $S_j^{(i)} = \hat{S}^{(i)} \cup \{ U_j^{(i)}, L_j^{(i)} \}$ \label{algInitActiveSet}

\For{$t \in \N$}{

Pull $\displaystyle I_t^{(i)} \in \argmax_{k \in S_j^{(i)} }  \hat{\mu}_k^{(i)}(t-1) + \sqrt{  \frac{\alpha \log(t) }{ T_k^{(i)}(t-1) }} $ \label{algPullArm}

\If{$t = A_j$}{

$\displaystyle B_j^{(i)} =  \argmax_{k \in S_j^{(i)} } T_k^{(i)}(A_j) - T_k^{(i)}(A_{j-1})$ (most played active arm in this phase) \label{algMostPlayed}

$\{ P_{j'}^{(i)} \}_{j'=j}^{\infty} \leftarrow \texttt{Update-Blocklist} ( \{ P_{j'}^{(i)} \}_{j'=j}^{\infty}  )$ (algorithm from \cite{chawla2020gossiping} performs no update; we propose using Algorithm \ref{algUpdateGossip2}) \label{algUpdateGossipDist}

$(H_j^{(i)} , R_j^{(i)}) = \texttt{Get-Rec}(i,j,P_j^{(i)})$ (see Algorithm \ref{algGetArm}) \label{algGetRec}

\eIf{$R_j^{(i)} \in S_j^{(i)}$ (recommendation already active)}{
    $S_{j+1}^{(i)} = S_j^{(i)}$ (same active set) \label{algSameActive}
}{

$\displaystyle U_{j+1}^{(i)} =  \argmax_{k \in \{ U_j^{(i)} , L_j^{(i)} \} } T_k^{(i)}(A_j) - T_k^{(i)}(A_{j-1})$ (most played non-sticky active arm) \label{algUjPlus1}

$L_{j+1}^{(i)} = R_j^{(i)}$ (replace least played non-sticky active arm with recommendation) \label{algLjPlus1}

$S_{j+1}^{(i)} = \hat{S}^{(i)} \cup \{ U_{j+1}^{(i)}, L_{j+1}^{(i)} \}$ (new active set) \label{algNewActive}

}

$j \leftarrow j+1$ (increment phase) \label{algIncrPhase}

}

}

\end{algorithm}

\begin{algorithm} \caption{$(H_j^{(i)} , R_j^{(i)}) = \texttt{Get-Rec}(i,j,P_j^{(i)})$ (black box to honest agents $i \in \{1,\ldots,n\}$)} \label{algGetArm}

\KwIn{Agent $i \in \{1,\ldots,n\}$, phase $j \in \N$, blocklist $P_j^{(i)}$}

Choose $H_j^{(i)}$ from $\{1,\ldots,n+m\} \setminus ( P_j^{(i)} \cup \{i\} )$ uniformly at random (i.e., from non-blocked agents), set $i' = H_j^{(i)}$ \label{algSampleRecommender}

\eIf{$i' \leq n$}{$R_j^{(i)} = B_j^{ (i') }$ (honest recommended most played) \label{algRecMostPlayed}}{Sample $R_j^{(i)}$ from $\nu_{j,i}^{(i')}$ (any $\mathcal{F}$-measurable map from $\Omega$ to the set of probability distributions over $\{1,\ldots,K\}$) \label{algRecArbitrary}
}

\KwReturn{$(H_j^{(i)}, R_j^{(i)})$}

\end{algorithm}

\section{Existing algorithm and lower bound} \label{secExistingAlgo}

The existing algorithm from \cite{chawla2020gossiping} is the special case of Algorithm \ref{algGeneral} where no blocking occurs, i.e., where $P_j^{(i)} = \emptyset\ \forall\ j \in \N, i \in \{1,\ldots,n\}$.\footnote{More precisely, we mean the synchronous algorithm in \cite{chawla2020gossiping}, which includes an asynchronous variant. For simplicity, we restrict attention to the former.} Thus, under our complete graph assumption, honest agent $i$ solicits a recommendation from an agent sampled uniformly from $\{1,\ldots,n+m\} \setminus \{i\}$ at each epoch. 

The following theorem lower bounds regret for this algorithm in the case of a single malicious agent ($m=1$). Note the lone malicious agent has index $n+1$ in this case. Also note we should not expect a nontrivial lower bound for \textit{any} strategy $\{ \nu_{j,i}^{(n+1)} \}_{j \in \N, i \in \{1,\ldots,n\}}$,  because (as discussed in Section \ref{secGeneralAlgo}) the malicious agent may behave as an honest agent, reducing the system to the setting of \cite{chawla2020gossiping}, for which regret is upper bounded by $O( S \log(T) / \Delta_2)$ (see \cite[Theorem 1]{chawla2020gossiping}). Hence, in Theorem \ref{thmNoBlacklist}, we consider an explicit (and extremely simple) strategy, where the malicious agent recommends uniformly random arms. Along these lines, note the theorem immediately extends to $m \in \{2,3,\ldots\}$, since we can assume $m-1$ malicious agents behave as honest ones, reducing the system to the setting of the theorem (with $n$ replaced by $n+m-1$).

\begin{thm} \label{thmNoBlacklist}
Assume $m=1$ and let $\nu_{j,i}^{(n+1)}$ be the uniform distribution over $\{1,\ldots,K\}$, for each $i \in \{1,\ldots,n\}, j \in \N$. Suppose each $i \in \{1,\ldots,n\}$ runs Algorithm \ref{algGeneral} with inputs $\alpha,\beta >1$ and performs no update in Line \ref{algUpdateGossipDist} (i.e., $i$ runs the algorithm from \cite{chawla2020gossiping}). Also assume $1 \in \cup_{i=1}^n \hat{S}^{(i)}$. Then for any $\varepsilon \in (0,1)$ independent of $T$ and any $i \in \{1,\ldots,n\}$,
\begin{equation}
\lim_{T \rightarrow \infty} \P \left( \frac{ R_T^{(i)} }{ \log T } \geq (1-\varepsilon) \alpha \left( 1 - \frac{1}{\sqrt{\alpha}} \right)^2   \sum_{k=2}^K \frac{ 1}{ \Delta_k } \right) = 1 ,
\end{equation}
and consequently,
\begin{equation}
\liminf_{T \rightarrow \infty} \frac{ \E R_T^{(i)} }{ \log T } \geq \alpha \left( 1 - \frac{1}{\sqrt{\alpha}} \right)^2 \sum_{k=2}^K \frac{1}{\Delta_k} .
\end{equation}
\end{thm}

\begin{rem} \label{remNoBlacklistSummary}
As an example, if $\Delta_2 = \cdots = \Delta_K = \Delta$ for some $\Delta \in (0,1)$, honest agents who run the algorithm from \cite{chawla2020gossiping} incur $\Omega( K \log(T) / \Delta)$ regret (with high probability and in expectation), equivalent to the single-agent $\text{UCB}(\alpha)$ baseline from \cite{auer2002finite}. Thus, \emph{the algorithm from \cite{chawla2020gossiping} fails when a single malicious agent is present}, in the sense that collaboration is no longer strictly beneficial. Notably, this occurs independently of the number of honest agents $n$.
\end{rem}

\begin{rem} \label{remBestInSticky}
We assume $1 \in \cup_{i=1}^n \hat{S}^{(i)}$ in Theorem \ref{thmNoBlacklist} to remove the trivial case where this assumption fails and agents incur $\Omega(T)$ regret. Note that, although we treat $\{ \hat{S}^{(i)} \}_{i=1}^n$ as deterministic, an alternative approach is to define them as $S$-sized uniformly random subsets; choosing $S = \ceil{ (K/n) \log(1/\varepsilon) }$ ensures this assumption holds with probability $1-\varepsilon$ (see \cite[Appendix L]{chawla2020gossiping}).
\end{rem}

\begin{proof}[Proof sketch]
The proof of Theorem \ref{thmNoBlacklist} is deferred to \cite[Appendix C]{vial2020robust}. At a high level, we separately consider three cases:
\begin{enumerate}
\item The best arm is not played often.
\item For some suboptimal arm $k \neq 1$ and all late phases $j$, $k$ is not active for $i$ during phase $j$.
\item The above cases fail, i.e., the best arm is played often and each suboptimal arm is active for $i$ at some late phase $j$.
\end{enumerate}
Our precise definition of the first case (see \eqref{eqFirstCase} below) implies that suboptimal arms are pulled polynomially many times, from which the theorem follows immediately. The second case occurs with vanishing probability owing to the uniformly random communication and malicious recommendations. For the third case, by definition, we can find a late time $t$ where arm $k$ is active but arm $1$ is pulled; by the $\text{UCB}(\alpha)$ policy (Line \ref{algPullArm} of Algorithm \ref{algGeneral}), this yields a lower bound on $T_k^{(i)}(t-1)$, which we use to prove the result in a manner similar to the single-agent bandit setting \cite{auer2002finite}.

More precisely, the first case is when the following occurs:
\begin{equation} \label{eqFirstCase}
\cup_{j \geq \Theta ( T^{1/\beta} )} \{ T_1^{(i)} ( A_{j-1} ) = o(A_{j-1}) \} \cup \cap_{ t = A_{j-1} + 1 }^{A_j}  \{ I_t^{(i)} \neq k \} .
\end{equation}
In words, $T_1^{(i)} ( A_{j-1} ) = o(A_{j-1})$ means the best arm has not been pulled a constant fraction of times before phase $j$, while $I_t^{(i)} \neq k\ \forall\ t \in \{ A_{j-1} + 1 , \ldots , A_j\}$ means this arm is never pulled within phase $j$. If the former occurs, then suboptimal arms are pulled $A_{j-1} - o ( A_{j-1} ) = \Theta ( A_{j-1} )$ times before phase $j$; since $A_{j-1} = \Theta ( j^{\beta} ) = \Theta ( T )$ (by definition and choice of $j$, respectively), this implies linear regret. Similarly, if the latter occurs, suboptimal arms are pulled $A_j - A_{j-1} = \Theta ( j^{\beta-1} ) = \Theta ( T^{(\beta-1)/\beta} )$ times during phase $j$, which gives polynomial regret. In both situations, the logarithmic lower bound on regret is immediate.

The second case occurs when the following holds:
\begin{equation} \label{eqSecondCase}
\cup_{k \neq 1} \cap_{j \geq \Theta ( T^{1/\beta} )} \{ k \notin S_j^{(i)} \}
\end{equation}
In this case, the key observation is that $R_j^{(i)} = k$ (i.e., $i$ is recommended arm $k$ at phase $j$) with probability at least $\frac{1}{n K}$. This holds because when no blocking occurs, $i$ contacts the malicious agent $n+1$ with probability $\frac{1}{n}$ at each phase $j$ (see Algorithm \ref{algGetArm}), who in turn recommends $k$ with probability $\frac{1}{K}$ (owing to the malicious strategy). Moreover, since the randomness in Algorithm \ref{algGetArm} and the malicious recommendations is independent across phases $j$, $\P ( \cap_{j \geq \Theta ( T^{1/\beta} )} \{ R_j^{(i)} \neq k  \} ) \rightarrow 0$ as $T \rightarrow \infty$. Finally, since $R_j^{(i)} \in S_{j+1}^{(i)}$ in Algorithm \ref{algGeneral}, the probability of \eqref{eqSecondCase} vanishes as well.

The third case is when the events \eqref{eqFirstCase} and \eqref{eqSecondCase} both fail. By definition of these events, for any $k \neq 1$, there exists a phase $j = \Theta ( T^{1/\beta} )$ and a time $t = \Theta ( T )$ during this phase such that $k \in S_j^{(i)}$, $T_1^{(i)}(t-1) = \Theta ( T )$, and $I_t^{(i)} = 1$. Hence, because $I_t^{(i)}$ is chosen according to the $\text{UCB}(\alpha)$ policy (Line \ref{algPullArm} of Algorithm \ref{algGeneral}),
\begin{equation}
\hat{\mu}_k^{(i)}(t-1) + \Theta \left( \sqrt{\frac{\log T}{ T_k^{(i)}(t-1) } } \right) \leq\hat{\mu}_1^{(i)}(t-1) + \Theta \left( \sqrt{ \frac{\log T}{T} } \right) .
\end{equation}
Since $\frac{\log T}{T} \rightarrow 0$ and $\hat{\mu}_1^{(i)}(t-1) - \hat{\mu}_k^{(i)}(t-1) \approx \mu_1 - \mu_k = \Delta_k$ with high probability due to concentration, the previous inequality implies $T_k^{(i)}(t-1) = \Omega (  \log(T) / \Delta_k^2 )$. This means $\Omega (  \log(T) / \Delta_k )$ regret from arm $k$. Summing over $k \neq 1$ completes the proof.
\end{proof}

\section{Proposed algorithm and upper bound} \label{secOurAlgo}

We next define our approach, which in words is quite simple: \textit{if agent $i'$ recommends arm $k$ at epoch $j-1$, and $k$ is not the most played arm in phase $j$, block $i'$ until epoch $j^{\eta}$}, where $\eta > 1$ is a tuning parameter. Hence, blocking depends only on the current phase $j$ and \textit{not} the number of bad arms that $i'$ has recommended in the past. More precisely, we propose running Algorithm \ref{algGeneral} with the \texttt{Update-Blocklists} subroutine defined in Algorithm \ref{algUpdateGossip2}.

\begin{algorithm} \caption{$\{ P_{j'}^{(i)} \}_{j'=j}^{\infty}= \texttt{Update-Blocklists}$ (executed by each honest agent $i \in \{1,\ldots,n\}$)} \label{algUpdateGossip2}

\If{$j > 1, B_j^{(i)} \neq R_{j-1}^{(i)}$ (previous recommendation not most played)}{ \label{algRecNotMostPlayed}

\For{$j' \in \{j,\ldots,\ceil{j^{\eta}}\}$}{

$P_{j'}^{(i)} \leftarrow P_{j'}^{(i)} \cup \{ H_{j-1}^{(i)} \}$ (block the recommender) \label{algRecToZero}

}

}

\KwReturn{$\{ P_{j'}^{(i)} \}_{j'=j}^{\infty}$}

\end{algorithm}

\begin{rem} \label{remKeyFeature}
The key feature of Algorithm \ref{algUpdateGossip2} is that the blocking period $\{j,\ldots,\ceil{j^{\eta}} \}$ grows with $j$. As mentioned in the introduction, this ensures two things. First, malicious agents who repeatedly recommend bad arms in late phases are blocked long enough to prevent the situation of the Theorem \ref{thmNoBlacklist} proof sketch (which causes $\Omega(K \log(T) / \Delta)$ regret). Second, honest agents $i'$ who are mistakenly blocked at early phases leave the blocklist soon enough to help spread the best arm to other honest agents. Note such mistakes can happen for three reasons: (1) $i'$ has much worse active arms than $i$, so any recommendation will perform poorly for $i$; (2) $i'$ has good active arms but accidentally recommends a bad arm (which will occur before time $\Theta ( \Delta^{-2} )$ \cite{audibert2010best}); (3) $i'$ recommends a good arm that performs poorly for $i$ (which also occurs before $\Theta ( \Delta^{-2} )$). See Theorem \ref{thmOurAlgorithmRegret} proof sketch and Remark \ref{remBlacklistPhen} for a more quantitative discussion of these ideas.
\end{rem}

\begin{rem} \label{remExpExp}
At a high level, malicious agents introduce a dilemma analogous to the standard MAB explore-exploit tradeoff: honest agents should block those who provide seemingly-bad recommendations -- analogous to pulling seemingly-bad arms less frequently, i.e., exploiting -- but should block mildly enough that honest agent mistakes are not punished too severely -- analogous to continued exploration of seemingly-bad arms. Thus, Remark \ref{remKeyFeature} and our analysis show that Algorithm \ref{algUpdateGossip2} provides the correct scaling ($j^{\eta}$-length blocking) for this additional explore-exploit tradeoff.
\end{rem}

\begin{rem}
We defined blocklists as infinite sequences to simplify the exposition; in practice, they can be maintained with $(m+n) \log T$ memory: $i$ can initialize $d^{(i)}(i') = 0$ and overwrite $d^{(i)} ( i' )$ with $\ceil{ j^{\eta} }$ if $i'$ is blocked at phase $j$ (for each $i'$), so that $P_j^{(i)} = \{ i' : d^{(i)}(i') \geq j \}$. Note $i$ requires $\log T$ memory to store, e.g., rewards, so this does not increase $i$'s storage cost in terms of $T$.
\end{rem}

Having defined our algorithm, we state a regret guarantee. We again assume $1 \in \cup_{i=1}^n \hat{S}^{(i)}$ (see Remark \ref{remBestInSticky}) but require no assumptions on the number of malicious agents or their strategies.

\begin{thm} \label{thmOurAlgorithmRegret}
Suppose each $i \in \{1,\ldots,n\}$ runs Algorithm \ref{algGeneral} with inputs $\beta > 1,\alpha > \frac{ 3 + (1 + \beta \eta ) / \beta }{ 2 }$ and uses  Algorithm \ref{algUpdateGossip2} as the \texttt{Update- Blocklists} subroutine with input $\eta > 1$. Also assume $1 \in \cup_{i=1}^n \hat{S}^{(i)}$. Then for any $i \in \{1,\ldots,n\}$ and any $T \in \N$,
\begin{equation} \label{eqOurAlgorithmRegret}
\E R_T^{(i)}  \leq 4 \alpha  \min \left\{ \frac{2 \eta - 1}{\eta-1} \sum_{k=2}^{m+3} \frac{1}{\Delta_k} + \sum_{k=m+4}^{S+m+4} \frac{1}{\Delta_k} ,  \sum_{k=2}^K \frac{1}{ \Delta_k }  \right\} \log T + C^{\star} , 
\end{equation}
where (by convention) $\Delta_k = 1\ \forall\ k > K$, and where $C^{\star}$ is a constant independent of $T$ defined in \eqref{eqDefnAdditiveConstant} in Appendix \ref{appUpperBound} and satisfying
\begin{align}
C^{\star} & = O \Big( (S/\Delta_2^2)^{ 2 \beta \eta / ( \beta - 1 ) } + S n K^2 + ( (m+n) \log n )^{\beta} \\
& \quad\quad\quad + ( K / \Delta_2 ) + m \log(K/\Delta_2) / \Delta_2 \Big) .
\end{align}
\end{thm}

\begin{rem} \label{remOurAlgorithmSummary}
Letting $S = \ceil{ (K/n) \log(1/\varepsilon) }$ (see Remark \ref{remBestInSticky}) and $\Delta_2 = \cdots = \Delta_K = \Delta$, Theorem \ref{thmOurAlgorithmRegret} shows regret scales as $O( \min \{ m+K/n , K \} \log(T) / \Delta )$ for our algorithm. Note this improves over the $O( K \log (T) / \Delta )$ regret of the single-agent baseline whenever $m < K (1-1/n)$. Thus, \emph{if the number of malicious agents is small compared to the number of arms, honest agents benefit from collaboration}. In contrast, even $m=1$ malicious agent nullifies this benefit for the existing algorithm (see Remark \ref{remNoBlacklistSummary}). This choice of $S$ does require knowledge of $n$, but knowledge of some lower bound $n' \leq n$ such that $m + K / n' = o(K)$ suffices. Equivalently, we can assume knowledge of $n+m$ (as in \cite{chawla2020gossiping}) and a lower bound on $n/(n+m)$ (e.g., honest agents know at least half of all agents are honest). Finally, we suspect the regret's linear dependence on $m$ is unavoidable, because malicious agents can behave like honest ones until late in the algorithm. Thus, $o(m)$ dependence requires blocking malicious agents while they are indistinguishable from honest ones, which increases blocking among honest agents and may prevent the best arm from spreading. 
\end{rem}

\begin{rem} \label{remMzero}
In the setting of Remark \ref{remOurAlgorithmSummary}, regret is $O( (K/n) \log(T)/\Delta )$ when $m=0$, which matches the $m=0$ regret from \cite{chawla2020gossiping}. We do have an additional multiplicative constant $(2\eta-1)/(\eta-1)$, but this can be removed by separately analyzing the cases $m=0$ and $m>0$ (see Remark \ref{remSeparateCases}). Our second-order term $C^\star$ is worse due to accidental blocking of malicious agents early in the algorithm. However, this seems inevitable for an algorithm that simultaneously works in the cases $m=0$ and $m>0$, without prior knowledge of the case.
\end{rem}

\begin{rem} \label{remMvsN}
Our algorithm can improve over the single-agent baseline even when $m \gg n$. For example, in the setting of Remark \ref{remOurAlgorithmSummary}, its regret is $O(K^{1-\lambda} \log(T)/\Delta)$ when $n \propto K^{\lambda}$ and $m \propto K^{ 1-\lambda }$ for some $\lambda \in (0,1)$. Note $m$ is polynomial in $n$ in this case, and the exponent can be made arbitrarily large by choosing $\lambda$ small. While stylized, this regime is interesting because honest agents are initially overwhelmed with malicious agent recommendations, which can be arbitrarily bad. Nevertheless, Theorem \ref{thmOurAlgorithmRegret} implies the best arm will eventually spread among honest agents, and honest agents will eventually block malicious ones. This is somewhat counterintuitive, as one may have expected us to need a bound on $m$ in terms of $n$ to bound regret.
\end{rem}

\begin{rem}
Our algorithm has two key parameters: $\eta$, which controls the blocking duration, and $\beta$, which controls the frequency of communication. In Theorem \ref{thmOurAlgorithmRegret}, we see the $\log T$ term decreases with $\eta$ but is independent of $\beta$. Intuitively, this means long-term regret is smaller when blocking is more aggressive, but is insensitive to the frequency of communication. The second-order term grows with $\eta$, because aggressive blocking delays the best arm from spreading among honest agents. In contrast, this term's dependence on $\beta$ is more complicated. On the one hand, $\Delta_2^{-4 \beta \eta / (\beta-1) }$ is the time before honest agents can reliably identify the best arm in a phase, which decreases as the phase length (i.e., as $\beta$) grows. On the other hand, the term $((m+n)\log n)^{\beta}$ is the additional time for the best arm to spread, which increases in $\beta$. See early regime in proof sketch for more details.
\end{rem}

\begin{rem}
By choosing the blocking parameter $\eta$ to be small and tightening the analysis, the $\Delta^{- 4 \beta \eta / (\beta-1) }$ dependence on the arm gap in Theorem \ref{thmOurAlgorithmRegret} can be improved and made close to $\Delta^{-2 \beta / (\beta-1)}$ (see \cite[Remark 13]{vial2020robust}), which matches the best known bound when $m=0$ \cite[Corollary 2]{chawla2020gossiping}. We note that improving this dependence to $\Delta^{-1}$, which would imply $O(\sqrt{T})$ regret for worst case $\Delta$, remains an open problem even without malicious agents.
\end{rem}

\begin{proof}[Proof sketch]
We prove Theorem \ref{thmOurAlgorithmRegret} in Appendix \ref{appUpperBound} but here describe the key ideas assuming $\Delta_2 = \cdots = \Delta_K = \Delta$ (to simplify the notation) and $S + m < K$ (so the theorem improves over the single-agent baseline). We first define a random phase $\tau$ such that 
\begin{equation} \label{eqTauMain}
1 \in S_j^{(i)},\ B_j^{(i)} = 1\ \forall\ i \in \{1,\ldots,n\}\ \forall\ j  \geq \tau ,
\end{equation}
i.e., the best arm is active and most played for all honest $i$ and all phases $j \geq \tau$ (see \eqref{eqDefnTau} in Appendix \ref{appUpperBound} for the formal definition). We then bound regret incurred in three regimes defined in terms of $\tau$.

{\bf Early regime:} This regime contains all phases before $\tau$, i.e., the first $A_{\tau}$ arm pulls. Since $A_{\tau} \triangleq \tau^{\beta}$, we can trivially bound regret in this regime by $\E \tau^{\beta}$. Our goal is to show $\E \tau^{\beta} < \infty$ as $T \rightarrow \infty$, so this regime only contributes to the constant $C^\star$. Toward this end, we first define a random phase $\tau_{stab} \leq \tau$ such that
\begin{equation} \label{eqTauStabMain}
1 \in S_j^{(i)} \Rightarrow B_j^{(i)} = 1\ \forall\ i \in \{1,\ldots,n\}\ \forall\ j \geq  \tau_{stab} ,
\end{equation}
i.e., the best arm is most played \textit{if} it is active (see \eqref{eqDefnTauStab} in Appendix \ref{appUpperBound}). Next, let $i^*$ be an honest agent with the best arm in its sticky set, i.e., $1 \in \hat{S}^{(i^*)}$ (such an agent exists by assumption). The key observation is that if $i \neq i^*$ contacts $i^*$ at phase $j \geq \tau_{stab}$, i.e., if $H_j^{(i)} = i^*$, then $i^*$ will (by definition of $\tau_{stab}$) recommend arm $1$, which $i$ will add to its active set $S_j^{(i)}$ (if not already present). Combined with \eqref{eqTauMain} and \eqref{eqTauStabMain}, this implies
\begin{equation} \label{eqTauToRec}
\tau \leq \max_{i \in \{1,\ldots,n\} \setminus \{i^* \} } \inf \{ j \geq \tau_{stab} : H_j^{(i)} = i^* \} .
\end{equation}
Now because of the uniform sampling in Algorithm \ref{algGetArm}, $H_j^{(i)} = i^*$ occurs every $O(m+n)$ phases on average, unless $i^*$ has been blocked. However, even if $i$ blocks $i^*$ just before $\tau_{stab}$, $i$ will un-block $i^*$ by phase $\tau_{stab}^{\eta}$, and $H_j^{(i)} = i^*$ will occur within $O(m+n)$ additional phases. This allows us to show that with high probability, $\inf \{ j \geq \tau_{stab} : H_j^{(i)} = i^* \} = O ( \tau_{stab}^{\eta} )$, where here $O(\cdot)$ hides $n$ and $m$. Combined with \eqref{eqTauToRec}, and bounding the maximum by a sum, we obtain $\E \tau^{\beta} \leq O ( \E \tau_{stab}^{\eta \beta} )$. Thus, it only remains to show $\E \tau_{stab}^{\eta \beta} < \infty$. This amounts to showing that if the best arm is active for phase $j$, it is most played within that phase, with high probability as $j \rightarrow \infty$. This in turn follows from classical results for best arm identification \cite{bubeck2011pure}, and the fact that the phase length $A_j - A_{j-1} = \Theta ( j^{\beta-1} )$ grows with $j$. We note the definition of $\tau_{stab}$ is taken from \cite{chawla2020gossiping}, but our analysis differs as we require a stronger result ($\E \tau_{stab}^{\eta \beta} < \infty$ instead of $\E \tau_{stab}^{\beta} < \infty$ in \cite{chawla2020gossiping}), owing to the fact that honest agents can mistakenly block one another in our algorithm.

{\bf Late regime:} The late regime (hereafter LR) contains phases $j \in \{ \max \{ T^{\phi} , \tau \} , \ldots , T^{1/\beta} \}$, where $\phi \in (0,1/\beta)$ will be chosen later (independent of $T$) and $T^{1/\beta}$ is the phase ending at time $A_{T^{1/\beta}} \triangleq T$. The key observation is that if malicious agent $i'$ recommends a suboptimal arm $k \neq 1$ to honest agent $i$ at such a phase $j$, $k$ will not be most played by $i$ (since $j \geq \tau$ and by definition $\tau$), so $i$ will block $i'$ until phase $j^{\eta}$. After phase $j^{\eta}$, $i'$ can again recommend a suboptimal arm, but $i$ will again block $i'$, this time until phase $j^{\eta^2}$. Iterating this argument, we see $i'$ can only recommend suboptimal arms at phases that scale as $\{ T^{\phi \eta^l} \}_{l=0}^{l_1}$, where $l_1 = - \log_{\eta}  (\beta \phi)$ indexes the last such phase in the LR (since $T^{ \phi \eta^{l_1} } = T^{1/\beta}$). Thus, \textit{irrespective of the horizon $T$}, each malicious agent can recommend only $l_1$ suboptimal arms in the LR. Combined with the fact that the LR begins at phase $\tau$ (after which honest agents only recommend the best arm), this means $i$ only explores $S+ l_1 m$ suboptimal arms during the LR: $S$ sticky arms and $l_1$ recommendations from each of $m$ malicious agents. Thus, the LR is roughly equivalent to an $(S+l_1 m)$-armed bandit. Using classical bounds from \cite{auer2002finite}, this implies that $i$ incurs $O( (S+l_1 m) \log(T) / \Delta )$ LR regret.

{\bf Intermediate regime:} The remaining phases $\tau, \ldots ,T^{\phi}$ are the intermediate regime (IR). Since this regime also starts after $\tau$, the argument from the LR shows that any malicious agent $i'$ can only recommend suboptimal arms at phases that scale as $\{ \tau^{\eta^l} \}_{l=0}^{ l_2 }$, where $l_2 = \log_{\eta} \log_{\tau} T^{1/\phi}$ is the last phase before the LR. However, since $\phi$ was assumed to be independent of $T$ in the LR, $l_2 \rightarrow \infty$ as $T \rightarrow \infty$, so the key result from the LR (that malicious $i'$ can only recommend finite suboptimal arms) fails. Hence, we concede that malicious agents may force $i$ to explore all suboptimal arms during in the IR. However, since the best arm is always active for $i$ and $t \leq A_{T^{\phi}} \triangleq T^{\phi \beta}$ in the IR, this is no worse than playing all $K$ arms for horizon $T^{\phi \beta}$, which means $O(K \log ( T^{\phi \beta} ) / \Delta)$ IR regret.

{\bf Finishing the proof:} In summary, we have argued
\begin{equation}
\E R_T^{(i)} = O(1) +  O ( K \log ( T^{\phi \beta} ) / \Delta) + O ( ( S + l_1 m ) \log(T) / \Delta ) ,
\end{equation}
where the three terms account for the early, intermediate, and late regimes, respectively. Choosing $\phi = 1 / ( K \beta )$ and recalling $l_1 = - \log_{\eta} ( \beta \phi ) = O ( \log K )$, we obtain
\begin{equation}
\E R_T^{(i)} = O ( ( S + m \log K )  \log(T) / \Delta ) .
\end{equation}
(Note this is worse than the bound reported in the theorem; in the actual proof, we tighten the analysis to avoid the $\log K$ factor.)
\end{proof}

\begin{rem} \label{remBlacklistPhen}
In short, our algorithm relies on three phenomena. First, for phases independent of $T$, polynomial-length blocking is mild enough that the best arm spreads (see early regime in proof sketch). Second, repeatedly blocking malicious agents means each recommends finitely many suboptimal arms at phases polynomial in $T$ (see late regime). Third, while blocking cannot eliminate malicious agents in between these regimes, the effective horizon $T^{1/K}$ is too small to appreciably increase regret (see intermediate regime).
\end{rem}

\begin{rem} \label{remActiveChanges}
In the absence of malicious agents, \cite[Proposition 1]{chawla2020gossiping} shows $1 \in S_{\tau'}^{(i)}$ and $S_j^{(i)} = S_{\tau'}^{(i)}\ \forall\ j \geq \tau'$, for some almost-surely finite phase $\tau'$; in words, active sets remain fixed after $\tau'$. This allows the authors to treat regret after phase $\tau'$ as in the single-agent bandit setting (with the actual set of arms $\{1,\ldots,K\}$ replaced by the fixed active set $S_{\tau'}^{(i)}$). With the introduction of malicious agents, active sets may change infinitely often as $T \rightarrow \infty$ (see intermediate/late regimes of proof sketch), which necessitates a more refined analysis.
\end{rem}

\section{Experiments} \label{secExperiments}

In this section, we illustrate our analysis with numerical results on synthetic and real datasets.

\subsection{Synthetic data}

For the arm means, we choose $\mu_1 = 0.95, \mu_2 = 0.85$ (so that $\Delta_2 = 0.1$) and sample $\mu_3, \ldots, \mu_K$ from $[0,0.85]$ uniformly. We fix $n = 25$, $\beta = 2$, and $\alpha = 4$. Note the existing algorithm has good empirical performance with similar parameters when $m=0$ (see \cite[Section 7]{chawla2020gossiping}). We choose $S = \ceil{K/n}$ (see Remark \ref{remBestInSticky}) and resample uniformly random sticky sets until $1 \in \cup_{i=1}^n  \hat{S}^{(i)}$. We consider two malicious agent strategies: a \textit{uniform strategy} and an \textit{omniscient strategy}, where $\nu_{j,i}^{(i')}$ is uniform over $\{1,\ldots,K\}$ and $\argmin_{ k' \in \{2,\ldots,K\} \setminus S_j^{(i)} } T_{k'}^{(i)} ( A_j )$, respectively.\footnote{\textit{Omniscient} refers to the fact that malicious agents exploit private information.} Note that the omniscient strategy recommends whichever inactive suboptimal arm has been played least thus far, which forces honest agents to continue exploring all suboptimal arms. 

In Figure \ref{figSynthetic}, we set $m=10, K = 100, \eta = 2$ and plot mean and standard deviation of regret over $50$ trials. We compare the proposed algorithm to the existing one from \cite{chawla2020gossiping}, a baseline with no communication between agents, and an oracle baseline where honest agents know and block malicious agents \textit{a priori}. Our algorithm performs closer to the oracle than the no communication baseline; the opposite is true for \cite{chawla2020gossiping}. Moreover, our algorithm incurs less than half the regret of the existing algorithm. This improvement occurs across various choices of $m$, $K$, and $\eta$; see \cite[Appendix D]{vial2020robust}. Results are roughly similar for the uniform and omniscient strategies. The most notable difference is the ``S-curve" for the existing algorithm in the latter case. We believe this occurs because the omniscient strategy more aggressively forces honest agents to play under-explored non-active arms. Being under-explored, these arms are likely to be played more than the best one in the subsequent phase, which causes honest agents to discard the best arm at $T \approx 6 \times 10^4$ and $T \approx 4 \times 10^4$ in Figures \ref{figSynthetic} and \ref{figReal}, respectively. This leads to the ``bump" in regret near those values of $T$.

\begin{figure}
\centering \includegraphics[width=2.9in]{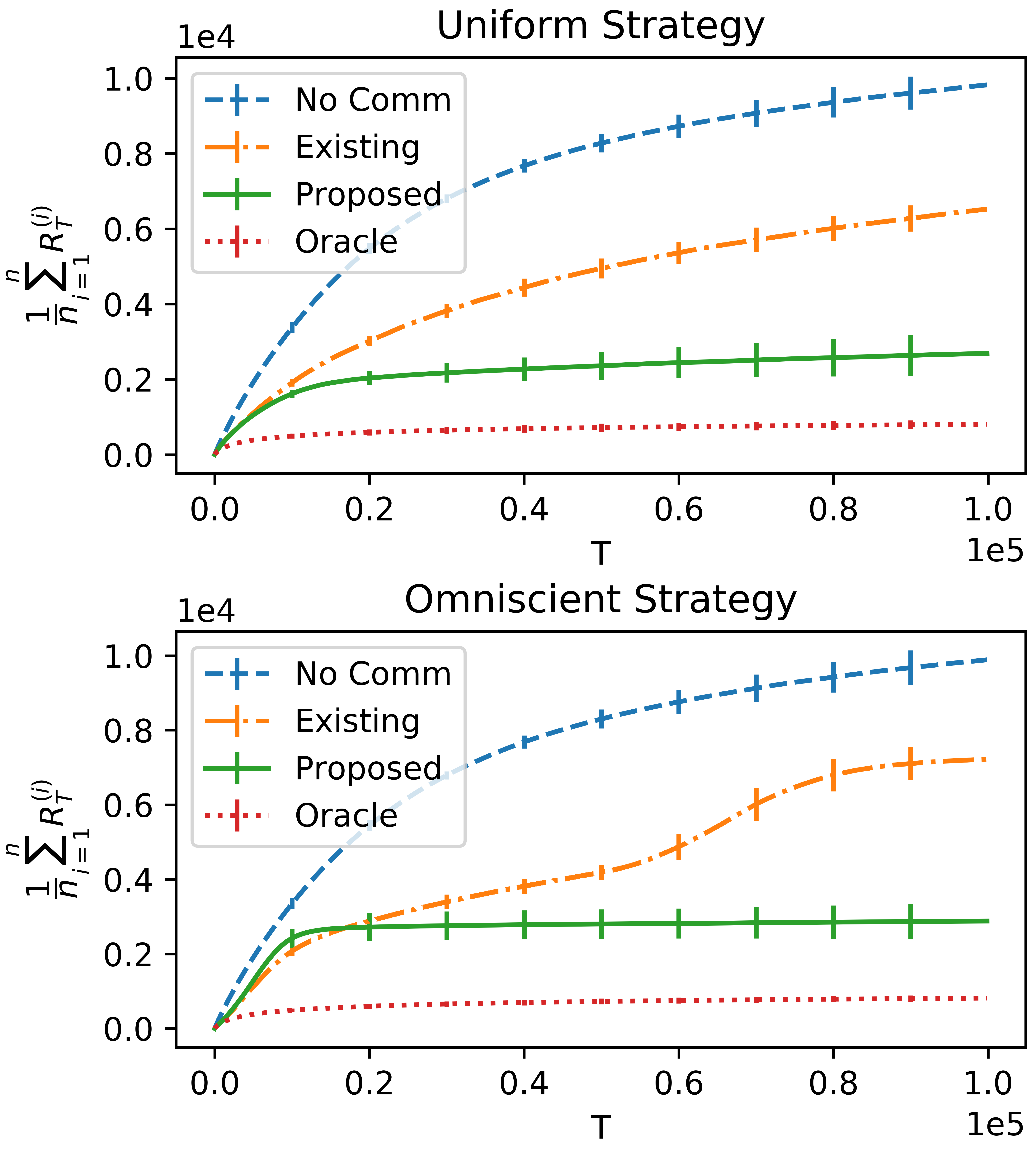} \caption{Regret for synthetic data, $m = 10, K = 100, \eta = 2$} \label{figSynthetic}
\end{figure}

\subsection{Real data}

For the same choices of $n$, $\beta$, $\alpha$, and malicious agent strategy, we test the four algorithms on the MovieLens dataset \cite{harper2015movielens}. We view movies as arms and derive arm means in a manner similar to \cite{chawla2020gossiping,sankararaman2019social}. First, we extract a matrix containing movie ratings by users with the same age, gender, and occupation, while also ensuring each user has rated $\geq 30$ movies and each movie has been rated $\geq 30$ times by the set of users. Next, we use matrix completion \cite{HaMaLeZa15} to estimate the missing entries of this matrix. From this estimated matrix, we map ratings to Bernoulli rewards by defining arm means as the fraction of ratings $\geq 4$ on a scale of $1$ to $5$ (i.e., we assume a user enjoyed a movie and gained a unit reward if he/she rated it $4$ or $5$ stars). Figure \ref{figReal} shows results similar to the synthetic case for $m=15, K = 100, \eta = 2$; \cite[Appendix D]{vial2020robust} again contains results for other $m$, $K$, and $\eta$ values.

\begin{figure}
\centering \includegraphics[width=2.9in]{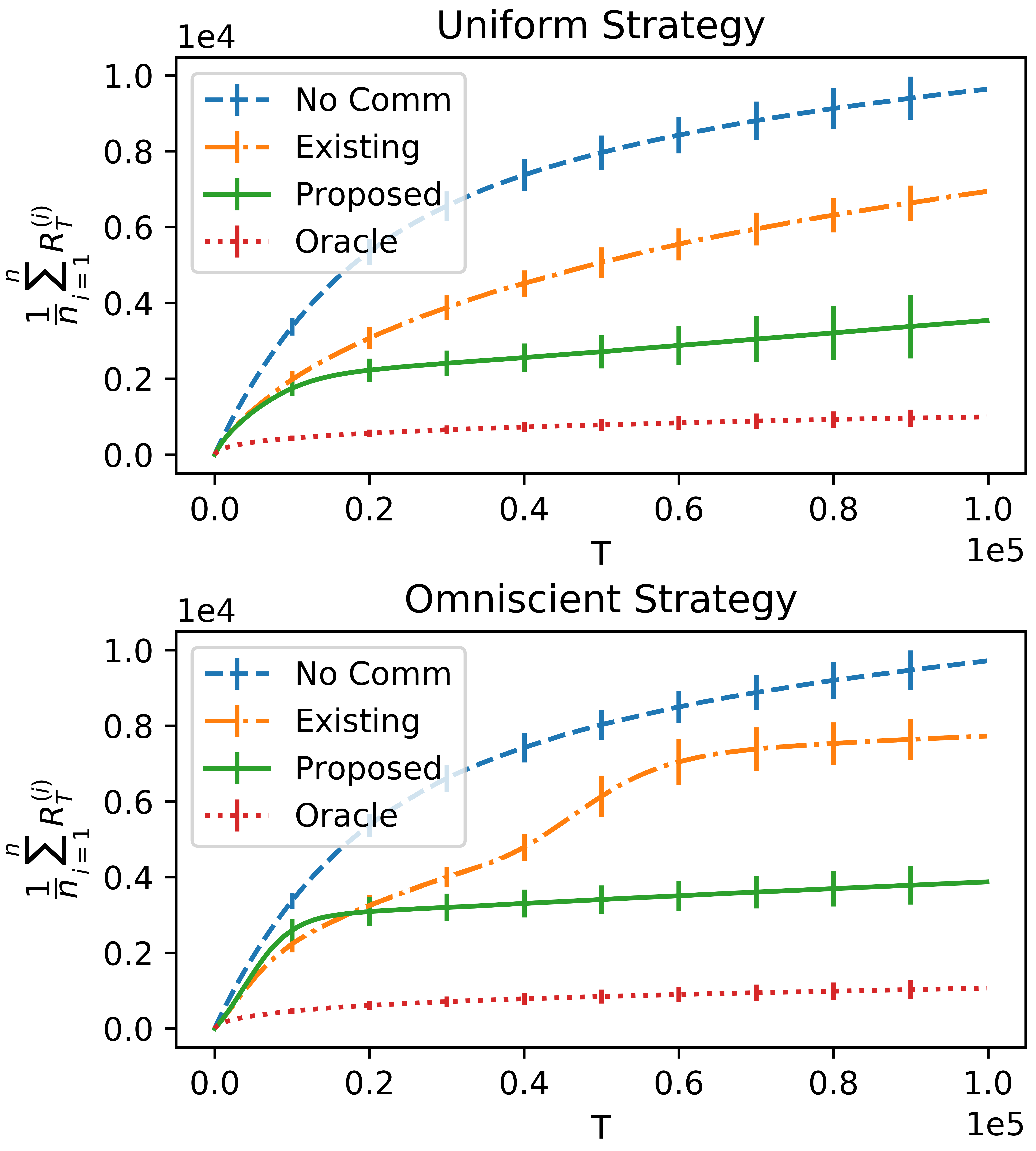} \caption{Regret for real data, $m = 15, K = 100, \eta = 2$} \label{figReal}
\end{figure}

\section{Conclusion} \label{secConclusions}

We studied a setting in which $n$ honest agents collaborate to minimize regret from a $K$-armed bandit and $m$ malicious agents disrupt this collaboration. We showed that even if $m=1$, existing algorithms fail to leverage the benefit of collaboration in this setting. We thus proposed an algorithm based on blocking. For the proposed algorithm, we showed regret is smaller than the single-agent baseline whenever $m$ is small compared to $K$, ensuring robustness against any malicious behavior.

\iftoggle{arxiv}{
\section*{Acknowledgments}

We thank Ronshee Chawla and Abishek Sankararaman for assistance with experiments. This work was partially supported by ONR Grant N00014-19-1-2566, NSF Grant SATC 1704778, NSF Grant CCF 1934986, ARO Grant W911NF-17-1-0359, and ARO Grant W911NF-19-1-0379.
}

\bibliographystyle{ACM-Reference-Format}
\bibliography{references}

\appendix

\section{Proof of Theorem \ref{thmOurAlgorithmRegret}} \label{appUpperBound}

Our proof uses a construction from \cite[Appendix B.1]{chawla2020gossiping} that we define here. First, let $\mathcal{S}^{(i)} = \{ W \subset \{1,\ldots,K\} : |W| = S+2, \hat{S}^{(i)} \cup \{ 1 \}  \subset W \}$ denote the $(S+2)$-sized sets of arms containing the sticky set $\hat{S}^{(i)}$ and the best arm $1$. For $i \in \{1,\ldots,n\}$, $j \in \N$, $W = \{ w_1, \ldots , w_{S+2} \} \in \mathcal{S}^{(i)}$, and $a = (a_1, \ldots , a_{S+2}) \in 
(\N \cup \{0\})^{S+2}$, let
\begin{align}
 \xi_j^{(i)} ( W , a )  = \{  S_j^{(i)} = W , & [ T_{w_l}^{(i)} ( A_{j-1} ) ]_{l = 1}^{S+2} = a, B_j^{(i)} \neq 1 \} 
\end{align}
be the event that honest agent $i$'s active set is $W$ at phase $j$, arm $w_l \in W$ was played $a_l$ times before phase $j$ began (for each $l$), and $1$ is \textit{not} the most played arm during phase $j$. Also define
\begin{equation}
\Xi_j^{(i)} = \cup_{W \in \mathcal{S}^{(i)}} \cup_{ a \in ( \N \cup \{0\} )^{S+2} } \xi_j^{(i)} ( W , a )
\end{equation}
to be the union (over active sets and histories of plays) of all such events. Let $\chi_j^{(i)} = 1 ( \Xi_j^{(i)} )$, where $1(\cdot)$ is the indicator function. Thus, $\chi_j^{(i)} = 0$ implies the best arm $1$ is not active for $i$ at phase $j$, or it is active and it is most played (under any history of plays). Using $\chi_j^{(i)}$, define the random variables
\begin{gather}
\textstyle \tau_{stab}^{(i)} = \inf \{ j \in \N : \chi_{j'}^{(i)} = 0\ \forall\ j' \geq j \} ,\\ 
\textstyle \tau_{stab} = \max_{i \in \{1,\ldots,n\} }  \tau_{stab}^{(i)} . \label{eqDefnTauStab}
\end{gather}
Thus, at the $\tau_{stab}$-th phase, and at all phases thereafter, the best arm $1$ will be the most played for any honest agent with this arm in its active set. Finally, let
\begin{gather}
\tau_{spr}^{(i)} = \inf \{ j \geq \tau_{stab} : 1 \in S_j^{(i)} \}  - \tau_{stab} , \\
\textstyle \tau_{spr} = \max_{i \in \{1,\ldots,n\} } \tau_{spr}^{(i)} , \quad \tau = \tau_{stab}+ \tau_{spr} . \label{eqDefnTau}
\end{gather}
Thus, at the $\tau$-th phase, and at all phases thereafter, the active set contains the best arm $1$ and this arm is most played, for all honest agents. Note the definition of $\tau$ implies the following property:
\begin{equation} \label{eqBestAfterTau}
1 \in S_j^{(i)},\ B_j^{(i)} = 1\ \forall\ j \in \{ \tau, \tau+1,\ldots\}\ \forall\ i \in \{1,\ldots,n\} .
\end{equation}
This holds inductively: $1 \in S_{\tau_{spr}^{(i)}+\tau_{stab}}^{(i)}$ by definition of $\tau_{spr}^{(i)}$, so $B_{\tau_{spr}^{(i)}+\tau_{stab}}^{(i)} = 1$ by definition of $\tau_{stab}$, so $1 \in S_{\tau_{spr}^{(i)}+\tau_{stab}+1}^{(i)}$ by Algorithm \ref{algGeneral}, etc. By definition of $\tau$, \eqref{eqBestAfterTau} follows.

Next, we let $\overline{S}^{(i)} = \{2,\ldots,K\} \cap \hat{S}^{(i)}$ and $\underline{S}^{(i)} = \{2,\ldots,K\} \setminus \hat{S}^{(i)}$ denote the suboptimal sticky and non-sticky arms for agent $i$, respectively, and we let $\gamma \in (0,1)$ be a constant to be chosen later. Then by upper bounding regret before $A_{\tau}$ as linear in time,
\begin{align} \label{eqRegretDecomp}
\E R_T^{(i)} \leq \E A_{\tau} & \textstyle + \sum_{k \in \overline{S}^{(i)}} \Delta_k \E \sum_{t=A_{\tau}+1}^T 1 ( I_t^{(i)} = k ) \\
&  \textstyle + \sum_{k \in \underline{S}^{(i)}} \Delta_k \E \sum_{t=A_{\tau}+1}^{A_{\ceil{T^{\gamma/\beta}}} \wedge T} 1 ( I_t^{(i)} = k ) \\
& \textstyle + \sum_{k \in \underline{S}^{(i)}} \Delta_k \E \sum_{t=A_{\ceil{T^{\gamma/\beta}} \vee \tau }+1}^T 1 ( I_t^{(i)} = k ) .
\end{align}
In words, the first term accounts for regret incurred at early times, i.e., before all agents are aware of the best arm and only recommend it moving forward. The remaining terms account for regret incurred from sticky arms at later times, from non-sticky arms at intermediate times, and from non-sticky arms at later times. The following lemmas bound these terms; see \cite[Appendix B]{vial2020robust} for proofs.

\begin{lem}[Early] \label{lemEarly}
For any $\beta > 1, \eta > 1, \alpha > \frac{ 3 + (1 + \beta \eta ) / \beta }{ 2 }$,
\begin{align}
\E A_{\tau} & \textstyle \leq 2^{1 + \beta \eta} \left( 4 + \left(\frac{26 \alpha (S+2)}{(\beta-1) \Delta_2^2}\right)^{2/(\beta-1)} \right) ^{\beta \eta}  \\
& \textstyle\quad + \frac{10 \beta}{ \beta-1} \max \{ 6 (m+n) \max \{ \log n , 2 (\beta-1) \} ,  3 ( 6^{\eta} + 2 ) \}^{\beta}  \\
& \textstyle \quad + \frac{ 2^{ \beta(2\alpha-3) +1 } n \binom{K}{2} (S+1)  }{ (2\alpha-3) (\beta(2\alpha-3) - 1) ( ( \beta (2\alpha-3)-1)/\eta - \beta )  }  .
\end{align}
\end{lem}
%\begin{proof}
%See Appendix \ref{appEarly}.
%\end{proof}

\begin{lem}[Late, sticky] \label{lemLateSticky}
\begin{align}
& \textstyle \sum_{k \in \overline{S}^{(i)}} \Delta_k \E \sum_{t=A_{\tau}+1}^T 1 ( I_t^{(i)} = k ) \\
&  \textstyle \leq  \sum_{k \in \overline{S}^{(i)}} \frac{4 \alpha \log T}{ \Delta_k } + 2  | \overline{S}^{(i)} | \sum_{t=1}^{\infty} t^{2(1-\alpha)} .
\end{align}
\end{lem}
%\begin{proof}
%See Appendix \ref{appLateSticky}.
%\end{proof}

\begin{lem}[Intermediate, non-sticky] \label{lemInterNonSticky}
For any $\gamma \in (0,1)$,
\begin{align}
& \textstyle \sum_{k \in \underline{S}^{(i)}} \Delta_k \E \sum_{t=A_{\tau}+1}^{A_{\ceil{T^{\gamma/\beta}}} \wedge T} 1 ( I_t^{(i)} = k ) \\
& \textstyle \leq \sum_{k \in \underline{S}^{(i)}}  \frac{4 \alpha \log ( A_{\ceil{T^{\gamma/\beta}}} \wedge T ) }{ \Delta_k }  + 2 |\underline{S}^{(i)} | \sum_{ t = 1 }^{ A_{ \ceil{ T^{ \gamma  / \beta } } } } t^{ 2(1-\alpha) } \\
& \textstyle \leq \frac{4 \alpha \gamma K \log T }{ \Delta_2 } + \frac{ 8 \alpha \beta K }{ \Delta_2 } + 2 |\underline{S}^{(i)} | \sum_{ t = 1 }^{ A_{ \ceil{ T^{ \gamma  / \beta } } } } t^{ 2(1-\alpha) }  .
\end{align}
\end{lem}
%\begin{proof}
%See Appendix \ref{appInterNonSticky}.
%\end{proof}

\begin{lem}[Late, non-sticky] \label{lemLateNonSticky}
For any $\gamma \in (0,1)$,
\begin{align}
& \textstyle \sum_{k \in \underline{S}^{(i)}} \Delta_k \E \sum_{t=A_{\ceil{T^{\gamma/\beta}} \vee \tau }+1}^T 1 ( I_t^{(i)} = k )  \\
& \textstyle \leq \frac{2 \eta-1}{\eta-1} \max_{ \tilde{S} \subset \underline{S}^{(i)} : |\tilde{S}| \leq m+2 } \sum_{k \in \tilde{S}} \frac{ 4 \alpha \log T }{ \Delta_k } \\
& \textstyle \quad +  \frac{ 8 \alpha \beta \log_{\eta}(1/\gamma) (m+2) }{\Delta_2} + 2 |\underline{S}^{(i)} | \sum_{ t = 1 + A_{ \ceil{ T^{ \gamma  / \beta } } } }^{\infty} t^{ 2(1-\alpha) } .
\end{align}
\end{lem}
%\begin{proof}
%See Appendix \ref{appLateNonSticky}.
%\end{proof}

\begin{rem} \label{remSeparateCases}
The multiplicative constant $(2\eta-1)/(\eta-1)$ in Theorem \ref{thmOurAlgorithmRegret} arises from Lemma \ref{lemLateNonSticky}. When $m=0$ (see Remark \ref{remMzero}), honest agents only recommend the best arm after $A_\tau$, so they do not play additional non-sticky arms. Hence, they do not incur the regret from Lemma \ref{lemLateNonSticky}, so this multiplicative constant can be removed.
\end{rem}

We next bound $\E R_T^{(i)} - \E A_{\tau}$ in each of two different cases. For the first case, we assume
\begin{equation}\label{eqRegretFirstCase}
\textstyle \frac{2 \eta - 1}{\eta-1} \sum_{k=2}^{m+3} \frac{1}{\Delta_k} + \sum_{k=m+4}^{S+m+4} \frac{1}{\Delta_k}  \leq \sum_{k=2}^K \frac{1}{\Delta_k}  .
\end{equation}
Set $\gamma = \Delta_2 / ( K \Delta_{S+m+4} ) \in (0,1)$. By Lemmas \ref{lemLateSticky} and \ref{lemLateNonSticky}, and the second bound from Lemma \ref{lemInterNonSticky},
\begin{align}
\E R_T^{(i)} - \E A_{\tau} & \textstyle \leq \frac{ 8 \alpha \beta ( K + \log_{\eta}(\frac{K}{\Delta_2}) (m+2) ) }{\Delta_2} \label{eqRegretSecondCaseSecondn} \\
& \textstyle \quad + 2 ( | \underline{S}^{(i)} | + | \overline{S}^{(i)} | ) \sum_{t=1}^{\infty} t^{2(1-\alpha)} ,  \\
&  \textstyle \quad + 4 \alpha  \Big( \frac{2 \eta-1}{\eta-1} \max_{ \tilde{S} \subset \underline{S}^{(i)} : |\tilde{S}| \leq m+2 } \sum_{k \in \tilde{S}} \frac{ 1 }{ \Delta_k }  \label{eqRegretSecondCaseFirstA}  \\ 
& \textstyle \quad\quad\quad\quad + \sum_{k \in \overline{S}^{(i)}} \frac{1}{ \Delta_k } + \frac{1}{\Delta_{S+m+4} } \Big) \log T \label{eqRegretSecondCaseFirstB}
\end{align}
where we also used $1/\gamma \leq K/\Delta_2$ in \eqref{eqRegretSecondCaseSecondn}. Note $| \underline{S}^{(i)} | + | \overline{S}^{(i)} | = K-1$ by definition, and
\begin{equation}\label{eqRegretBoundIntegralApprox}
\textstyle \sum_{t=1}^{\infty} t^{2(1-\alpha)} < 1 + \int_{t=1}^{\infty} t^{2(1-\alpha)} dt  = \frac{2 (\alpha - 1 )}{2 \alpha-3} . 
\end{equation}
Also, since $\Delta_k \leq \Delta_{k+1}$ and $\eta > 1$, the term in parentheses in  \eqref{eqRegretSecondCaseFirstA}-\eqref{eqRegretSecondCaseFirstB} is maximized if $\{2,\ldots,m+3\} \subset \underline{S}^{(i)}$ and $\overline{S}^{(i)} = \{ m +4 , \ldots , S+m+3 \}$ (note \eqref{eqRegretFirstCase} ensures $S+m+3 \leq K$). Therefore,
\begin{align}
& \textstyle \frac{2 \eta-1}{\eta-1} \max_{ \tilde{S} \subset \underline{S}^{(i)} : |\tilde{S}| \leq m+2 } \sum_{k \in \tilde{S}} \frac{ 1 }{ \Delta_k }  + \sum_{k \in \overline{S}^{(i)}} \frac{1}{ \Delta_k } + \frac{1}{\Delta_{S+m+4} }  \\
& \textstyle \leq \frac{2 \eta - 1}{\eta-1} \sum_{k=2}^{m+3} \frac{1}{\Delta_k} + \sum_{k=m+4}^{S+m+4} \frac{1}{\Delta_k} .
\end{align}
Combining, we have shown that if \eqref{eqRegretFirstCase} holds,
\begin{align} \label{eqRegretFirstCaseBound} 
\E R_T^{(i)} - \E A_{\tau} & \textstyle \leq 4 \alpha  \left( \frac{2 \eta - 1}{\eta-1} \sum_{k=2}^{m+3} \frac{1}{\Delta_k} + \sum_{k=m+4}^{S+m+4} \frac{1}{\Delta_k} \right) \log T  \\
& \textstyle \quad + \frac{ 8 \alpha \beta ( K + \log_{\eta}(K/\Delta_2) (m+2) ) }{\Delta_2} + \frac{4 K (\alpha-1)}{ 2 \alpha - 3 } .
\end{align}
If instead \eqref{eqRegretFirstCase} fails, choose any $\gamma \in ( \log(T-1) / \log(T) , 1 )$. Then $A_{ \ceil{ T^{\gamma / \beta } } } = \ceil{ \ceil{ T^{\gamma / \beta } }^{\beta} } \geq T^{\gamma} > T-1$, so $A_{ \ceil{ T^{\gamma / \beta } } } \geq T$, and the final term in \eqref{eqRegretDecomp} is zero. Moreover, $A_{ \ceil{ T^{\gamma / \beta } } } \wedge T = T$ by choice of $\gamma$. Then by Lemma \ref{lemLateSticky} and the first bound in Lemma \ref{lemInterNonSticky}, and an integral approximation like \eqref{eqRegretBoundIntegralApprox},
\begin{equation} \label{eqRegretSecondCaseBound}
\textstyle \E R_T^{(i)} - \E A_{\tau}  \leq 4 \alpha \log (T) \sum_{k=2}^K \frac{1}{ \Delta_k } + \frac{ 4 K ( \alpha - 1 )}{2 \alpha-3}  .  
\end{equation}
To summarize, we showed \eqref{eqRegretFirstCaseBound} holds if \eqref{eqRegretFirstCase} holds and \eqref{eqRegretSecondCaseBound} holds if \eqref{eqRegretFirstCase} fails. The theorem follows by plugging in the estimate for $\E A_{\tau}$ from Lemma \ref{lemEarly} and defining the constant \begin{align} \label{eqDefnAdditiveConstant}
C^{\star} & \textstyle = 2^{1 + \beta \eta} \left( 4 + \left(\frac{26 \alpha (S+2)}{(\beta-1) \Delta_2^2}\right)^{2/(\beta-1)} \right) ^{\beta \eta}  \\
& \textstyle\quad + \frac{ 2^{ \beta(2\alpha-3) +1 } n \binom{K}{2} (S+1)  }{ (2\alpha-3) (\beta(2\alpha-3) - 1) ( ( \beta (2\alpha-3)-1)/\eta - \beta )  } \\
& \textstyle\quad + \frac{10 \beta}{ \beta-1} \max \{ 6 (m+n) \max \{ \log n , 2 (\beta-1) \} ,  3 ( 6^{\eta} + 2 ) \}^{\beta}  \\
& \textstyle\quad + \frac{ 4 K ( \alpha - 1 )}{2 \alpha-3}  + \frac{ 8 \alpha \beta ( K + \log_{\eta}(K /\Delta_2) (m+2) ) }{\Delta_2} \\
& \textstyle= O \Big( ( \frac{S}{\Delta_2^2} )^{ 2 \beta \eta / ( \beta - 1 ) } + S n K^2 + ( (m+n) \log n )^{\beta} \\
& \textstyle \quad\quad\quad + \frac{K}{\Delta_2} + \frac{m}{\Delta_2} \log \frac{K}{\Delta_2} \Big) .
\end{align}

\iftoggle{arxiv}{

\newpage \onecolumn

\section{Proofs of lemmas for Theorem \ref{thmOurAlgorithmRegret}} \label{appUpperBoundLemmas}

We begin by proving Lemmas \ref{lemLateSticky} and \ref{lemInterNonSticky} in Appendices \ref{appLateSticky} and \ref{appInterNonSticky}, respectively, which modify standard arguments from the single-agent setting  \cite{auer2002finite} to account for random sticky sets. We then prove Lemma \ref{lemLateNonSticky} in Appendix \ref{appLateNonSticky}, which builds on these arguments but also requires delicate bounds to cope with worst-case malicious agent recommendations and to ensure such agents are blocked. Lastly, we prove Lemma \ref{lemEarly} in Appendix \ref{appEarly}, which leverages a result from the cooperative setting \cite{chawla2020gossiping} but requires nontrivial modification due to accidental blocking among honest agents. To avoid cluttering these proofs, we defer some proofs that tedious calculations to Appendix \ref{appAuxiliary}. Moving forward, we define $A^{-1}(t) = \inf \{ j \in \N: t \leq A_j  \}$. Note $I_t^{(i)} \in S_{A^{-1}(t)}^{(i)}$, i.e., at time $t$, agent $i$ chooses an arm from $S_{A^{-1}(t)}^{(i)} \subset \{1,\ldots,K\}$.

\subsection{Late regret from sticky arms (proof of Lemma \ref{lemLateSticky})} \label{appLateSticky}

We bound the number of pulls of $k \in \overline{S}^{(i)}$ using ideas from \cite{auer2002finite}. First, we write
\begin{align}
\E \sum_{t=A_{\tau}+1}^T 1 ( I_t^{(i)} = k ) & = \E \sum_{t=A_{\tau}+1}^T 1 \left( I_t^{(i)} = k , T_k^{(i)}(t-1) < \frac{4 \alpha \log T}{\Delta_k^2} \right) \label{eqLateStickyAsTerm} \\
& \quad\quad + \E \sum_{t=A_{\tau}+1}^T 1 \left( I_t^{(i)} = k , T_k^{(i)}(t-1) \geq \frac{4 \alpha \log T}{\Delta_k^2} \right) . \label{eqLateStickyConcTerm}
\end{align}
By definition $T_k^{(i)}(t-1) = \sum_{s=1}^{t-1} 1 ( I_s^{(i)} = k) $, we can bound \eqref{eqLateStickyAsTerm} by observing that, almost surely,
\begin{equation} \label{eqLateStickyAsBound}
\sum_{t=A_{\tau}+1}^T 1 \left( I_t^{(i)} = k , T_k^{(i)}(t-1) < \frac{4 \alpha \log T}{\Delta_k^2} \right) \leq \frac{4 \alpha \log T}{ \Delta_k^2 } . 
\end{equation}
To bound \eqref{eqLateStickyConcTerm}, we first note
\begin{align}
& \E \sum_{t=A_{\tau}+1}^T 1 \left( I_t^{(i)} = k , T_k^{(i)}(t-1) \geq \frac{4 \alpha \log T}{\Delta_k^2} \right) = \sum_{t=1}^T \P \left( t > A_{\tau} , I_t^{(i)} = k , T_k^{(i)}(t-1) \geq \frac{4 \alpha \log T}{\Delta_k^2} \right) .
\end{align}
Now let $t \in \{1,\ldots,T\}$. Note $t > A_{\tau}$ implies $A^{-1}(t) > \tau$ by definition of $A^{-1}$, which by \eqref{eqBestAfterTau} implies $1 \in S_{A^{-1}(t)}^{(i)}$ (i.e., $i$ is aware of arm $1$ at $t$). Thus, $t > A_{\tau}$ and $I_t^{(i)} = k$ imply agent $i$ chose arm $k$ over arm $1$ at time $t$, which implies
\begin{equation}
\bar{X}_{1, T_1^{(i)} ( t - 1 ) }^{(i)} + c_{t , T_1^{(i)}(t-1) } \leq \bar{X}_{k, T_k^{(i)} ( t - 1 ) }^{(i)} + c_{t , T_k^{(i)}(t-1) } ,
\end{equation}
where $\bar{X}_{k,s}^{(i)}$ is the average of $s$ independent $\text{Bernoulli}(\mu_k)$ random variables and $c_{t,s} = \sqrt{ \alpha \log(t) / s }$. Thus,
\begin{align}
\P \left( t > A_{\tau} , I_t^{(i)} = k , T_k^{(i)}(t-1) \geq \frac{4 \alpha \log T}{\Delta_k^2} \right) & \leq \P \left( \bar{X}_{1, T_1^{(i)} ( t - 1 ) }^{(i)} + c_{t , T_1^{(i)}(t-1) } \leq \bar{X}_{k, T_k^{(i)} ( t - 1 ) }^{(i)} + c_{t , T_k^{(i)}(t-1) } , T_k^{(i)}(t-1) \geq \frac{4 \alpha \log T}{\Delta_k^2} \right) \\
& \leq 2 t^{ 2(1-\alpha) } , \label{eqLateStickyClassicalBound}
\end{align}
where the second inequality is the classical bound from \cite{auer2002finite}. Substituting into \eqref{eqLateStickyConcTerm},
\begin{equation}\label{eqLateStickyConcBound}
\E \sum_{t=A_{\tau}+1}^T 1 \left( I_t^{(i)} = k , T_k^{(i)}(t-1) \geq \frac{4 \alpha \log T}{\Delta_k^2} \right)  \leq 2 \sum_{t = 1}^{\infty} t^{2(1-\alpha)} .
\end{equation}
Finally, plugging \eqref{eqLateStickyConcBound} into \eqref{eqLateStickyConcTerm} and \eqref{eqLateStickyAsBound} into \eqref{eqLateStickyAsTerm} and summing over $k$ completes the proof.

\subsection{Intermediate regret from non-sticky arms (proof of Lemma \ref{lemInterNonSticky})} \label{appInterNonSticky}

The first bound follows by replacing $T$ with $A_{ \ceil{ T^{\gamma/\beta} } } \wedge T$ and $\overline{S}^{(i)}$ with $\underline{S}^{(i)}$ in the proof of Lemma \ref{lemLateSticky}, but otherwise repeating the same arguments. For the second bound, first note
\begin{equation}\label{eqPhaseBound}
A_{ \ceil{ T^{ \gamma  / \beta }  }} = \ceil{ \ceil{ T^{ \gamma  / \beta } }^{\beta} } \leq  ( T^{ \gamma  / \beta } + 1 )^{\beta} + 1 \leq 2^{\beta+1} T^{\gamma} < e^{2 \beta } T^{\gamma} ,
\end{equation}
where the first inequality is $\ceil{x} \leq x+1$, the second uses $T \geq 1$, and the third uses $\beta > 1$. Therefore,
\begin{equation}
\log ( A_{\ceil{T^{\gamma/\beta}}} \wedge T ) \leq \log ( A_{\ceil{T^{\gamma/\beta}}} ) < \gamma \log(T) + 2 \beta .
\end{equation}
Combined with the inequalities $\Delta_2 \leq \Delta_k$ and $| \underline{S}^{(i)} | < K$, we thus obtain
\begin{equation}
\sum_{k \in \underline{S}^{(i)}}  \frac{4 \alpha \log ( A_{\ceil{T^{\gamma/\beta}}} \wedge T ) }{ \Delta_k } < \frac{ 4 \alpha K ( \gamma \log (T) + 2 \beta ) }{\Delta_2} = \frac{4 \alpha \gamma K \log T }{ \Delta_2 } + \frac{ 8 \alpha \beta K }{ \Delta_2 } .
\end{equation}

 \subsection{Late regret from non-sticky arms (proof of Lemma \ref{lemLateNonSticky})} \label{appLateNonSticky}

For each $k \in \underline{S}^{(i)}$ and each $l, t \in \N$, define the random variables
\begin{align}
X_{k,l,t} =1 \left( t > A_{\tau} , I_t^{(i)} = k , T_k^{(i)}(t-1) \geq \frac{ 4 \alpha \log ( A_{ \ceil{ T^{ \gamma \eta^l / \beta }  } } \wedge T )  }{ \Delta_k^2 }  \right)  , \quad Y_{k,l,t} = 1 \left( t > A_{\tau} , I_t^{(i)} = k , T_k^{(i)}(t-1) < \frac{ 4 \alpha \log ( A_{ \ceil{ T^{ \gamma \eta^l / \beta }  } } \wedge T ) }{ \Delta_k^2 }  \right) .
\end{align}
We can then rewrite the number of pulls of arm $k \in \underline{S}^{(i)}$ after time $A_{\ceil{T^{\gamma/\beta}} \vee \tau }$ as
\begin{align} \label{eqLateNonStickyDoubleSum} 
\sum_{t=A_{\ceil{T^{\gamma/\beta}} \vee \tau }+1}^T 1 ( I_t^{(i)} = k )  = \sum_{t=A_{\ceil{T^{\gamma/\beta}} }+1}^T 1 ( t > A_{\tau} , I_t^{(i)} = k ) 
 = \sum_{l=1}^{ \ceil{ \log_{\eta}(1/\gamma) } } \sum_{t = 1 + A_{ \ceil{ T^{ \gamma \eta^{l-1} / \beta } } } }^{ A_{ \ceil{ T^{ \gamma \eta^l / \beta }  } } \wedge T } ( X_{k,l,t} + Y_{k,l,t} ) . 
\end{align}
(Note the first term of the double summation on the right expression corresponds to time $1 + A_{ \ceil{ T^{ \gamma \eta^{1-1} / \beta } } } = 1 + A_{ \ceil{ T^{ \gamma / \beta } } }$, and the final term corresponds to time
\begin{equation}\label{eqLateNonStickyFinalAfterT}
A_{ \ceil{ T^{ \gamma \eta^{ \ceil{\log_{\eta}(1/\gamma)}} / \beta } } } \wedge T = \ceil{ \ceil{ T^{ \gamma \eta^{ \ceil{\log_{\eta}(1/\gamma)}} / \beta } }^{\beta} } \wedge T  = T ,
\end{equation}
so all summands in the middle expression are accounted for in the right expression of \eqref{eqLateNonStickyDoubleSum}.) Therefore,
\begin{align}
\sum_{k \in \underline{S}^{(i)}} \Delta_k \E \sum_{t=A_{\ceil{T^{\gamma/\beta}} \vee \tau }+1}^T 1 ( I_t^{(i)} = k )  & \leq \sum_{k \in \underline{S}^{(i)}} \Delta_k \sum_{l=1}^{ \ceil{ \log_{\eta}(1/\gamma) } } \sum_{t = 1 + A_{ \ceil{ T^{ \gamma \eta^{l-1} / \beta } } } }^{ A_{ \ceil{ T^{ \gamma \eta^l / \beta }  } } \wedge T } \P ( X_{k,l,t} = 1 ) \quad \label{eqLateNonStickyConcTerm} \\
& \quad + \E \sum_{l=1}^{ \ceil{ \log_{\eta}(1/\gamma) } }  \sum_{k \in \underline{S}^{(i)}} \Delta_k  \sum_{t = 1 + A_{ \ceil{ T^{ \gamma \eta^{l-1} / \beta } } } }^{ A_{ \ceil{ T^{ \gamma \eta^l / \beta }  } }  } Y_{k,l,t} . \quad \label{eqLateNonStickyAsTerm} 
\end{align}
(Note we also used $A_{ \ceil{ T^{ \gamma \eta^l / \beta }  } } \wedge T \leq A_{ \ceil{ T^{ \gamma \eta^l / \beta }  } }$ in \eqref{eqLateNonStickyAsTerm}.) We next bound \eqref{eqLateNonStickyConcTerm}. Choose any $k \in \underline{S}^{(i)}$, $l \in \{ 1 , \ldots , \ceil{ \log_{\eta}(1/\gamma) } \}$, and $t \in \{  1 + A_{ \ceil{ T^{ \gamma \eta^{l-1} / \beta } } } , \ldots , A_{ \ceil{ T^{ \gamma \eta^l / \beta } } } \wedge T \}$. Then following the argument leading to \eqref{eqLateStickyClassicalBound} in the proof of Lemma \ref{lemLateSticky},
\begin{equation}\label{eqLateNonStickyClassicalBound}
\P ( X_{k,l,t} = 1 ) = \P \left( t > A_{\tau} , I_t^{(i)} = k , T_k^{(i)}(t-1) \geq \frac{ 4 \alpha \log ( A_{ \ceil{ T^{ \gamma \eta^l / \beta }  } } \wedge T ) }{ \Delta_k^2 }  \right) \leq 2 t^{2(1-\alpha)} .
\end{equation}
(To be precise, $T$ should be replaced by $A_{ \ceil{ T^{ \gamma \eta^l / \beta }  } } \wedge T$ in this argument; the same argument then applies since we are considering $t \leq A_{ \ceil{ T^{ \gamma \eta^l / \beta }  } } \wedge T$.) We thus obtain the following bound for \eqref{eqLateNonStickyConcTerm}:
\begin{equation}\label{eqLateNonStickyConcBound}
\sum_{k \in \underline{S}^{(i)}} \Delta_k \sum_{l=1}^{ \ceil{ \log_{\eta}(1/\gamma) } } \sum_{t = 1 + A_{ \ceil{ T^{ \gamma \eta^{l-1} / \beta } } } }^{ A_{ \ceil{ T^{ \gamma \eta^l / \beta }  } } \wedge T } \P ( X_{k,l,t} = 1 ) \leq 2 |\underline{S}^{(i)} | \sum_{ t = 1 + A_{ \ceil{ T^{ \gamma  / \beta } } } }^{\infty}  t^{ 2(1-\alpha) } .
\end{equation}

To bound \eqref{eqLateNonStickyAsTerm}, we begin with two key claims. The first claim roughly says that if arm $k$ is pulled at time $A_{T^{ \gamma \eta^{l-1} / \beta  }} < t \leq A_{T^{ \gamma \eta^{l} / \beta }}$ (which occurs if $Y_{k,l,t} = 1$), then $k$ must have been active at some phase between $T^{ \gamma \eta^{l-1} / \beta }$ and $T^{ \gamma \eta^l / \beta }$. Thus, this claim is rather obvious; the only subtlety is that the indicator function in \eqref{eqLateNonStickyKeyImp_2} does not depend on $t$, which will be crucial later (see \eqref{eqClmLateNonStickyIndicatorApp} below).
\begin{clm} \label{clmLateNonStickyIndicator}
For any $l \in \{1,\ldots, \ceil{ \log_{\eta}(1/\gamma) }$, $k \in \underline{S}^{(i)}$, and $t \in \{ 1 + A_{ \ceil{ T^{ \gamma \eta^{l-1} / \beta } } } , \ldots , A_{ \ceil{ T^{ \gamma \eta^l / \beta } } } \}$,
\begin{equation}\label{eqLateNonStickyKeyImp_2}
Y_{k,l,t} =  1 \left( k \in \cup_{j = \ceil{ T^{ \gamma \eta^{l-1} / \beta } } \vee \tau }^{ \ceil{ T^{ \gamma \eta^l / \beta } } } S_{j}^{(i)} \right) Y_{k,l,t}\ a.s.
\end{equation}
\end{clm}
\begin{proof}
Fix $l,k,t$. Recall $Y_{k,l,t}$ is binary-valued, so it suffices to show
\begin{equation}\label{eqLateNonStickyKeyImp}
Y_{k,l,t} = 1 \quad \Rightarrow \quad k \in \cup_{j = \ceil{ T^{ \gamma \eta^{l-1} / \beta } } \vee \tau }^{ \ceil{ T^{ \gamma \eta^l / \beta } } } S_{j}^{(i)}  .
\end{equation}
We prove \eqref{eqLateNonStickyKeyImp} by contradiction: assume instead that $Y_{k,l,t} = 1$ and
\begin{equation}\label{eqLateNonStickyKeyImpCont}
k \notin \cup_{j = \ceil{ T^{ \gamma \eta^{l-1} / \beta } } \vee \tau }^{ \ceil{ T^{ \gamma \eta^l / \beta } } } S_{j}^{(i)} .
\end{equation}
Recall $Y_{k,l,t} = 1$ implies $t > A_{\tau}$ by definition of $Y_{k,l,t}$; since $t \in \{ 1 + A_{ \ceil{ T^{ \gamma \eta^{l-1} / \beta } } } , \ldots , A_{ \ceil{ T^{ \gamma \eta^l / \beta } } } \}$ in the statement of the claim, we conclude
\begin{equation}
t \in \{ 1 + A_{ \ceil{ T^{ \gamma \eta^{l-1} / \beta } } \vee \tau } , \ldots , A_{ \ceil{ T^{ \gamma \eta^l / \beta } } } \} .
\end{equation}
It follows by definition of $A^{-1}$ that $A^{-1}(t) \in \{ \ceil{ T^{ \gamma \eta^{l-1} / \beta } } \vee \tau , \ldots , \ceil{ T^{ \gamma \eta^l / \beta } } \}$, so
\begin{equation}\label{eqLateNonStickyKeyImpCont_2}
S_{A^{-1}(t)}^{(i)} \subset \cup_{j = \ceil{ T^{ \gamma \eta^{l-1} / \beta } } \vee \tau }^{ \ceil{ T^{ \gamma \eta^l / \beta } } } S_{j}^{(i)} .
\end{equation}
Comparing \eqref{eqLateNonStickyKeyImpCont} and \eqref{eqLateNonStickyKeyImpCont_2} shows $k \notin S_{A^{-1}(t)}$ (i.e., $k$ is not an active arm at time $t$); this implies $I_t^{(i)} \neq k$ (i.e., $k$ is not pulled at time $t$), contradicting $Y_{k,l,t} = 1$ by definition.
\end{proof}

The second claim bounds the sum of inverse arm gaps for suboptimal non-sticky arms pulled between phases $T^{\gamma \eta^{l-1} / \beta}$ and $T^{\gamma \eta^l / \beta}$. The idea is that each of $m$ malicious agents can only recommend one such arm between these phases (since if this recommendation occurs at phase $j \geq T^{\gamma \eta^{l-1} / \beta}$, the agent is blocked until $j^{\eta} \geq T^{\gamma \eta^l / \beta}$). Similar to the previous claim, the upper bound is uniform across $l$, which is crucial in its application \eqref{eqClmLateNonStickyIndicatorSumApp}.
\begin{clm} \label{clmLateNonStickyIndicatorSum}
For any $l \in \{1,\ldots,\ceil{ \log_{\eta}(1/\gamma) } \}$, 
\begin{equation}
\sum_{k \in \underline{S}^{(i)}}  \frac{ 1   }{ \Delta_k } 1 \left( k \in \cup_{j = \ceil{ T^{ \gamma \eta^{l-1} / \beta } } \vee \tau }^{ \ceil{ T^{ \gamma \eta^l / \beta } } } S_{j}^{(i)} \right) \leq \max_{ \tilde{S} \subset \underline{S}^{(i)} : |\tilde{S}| \leq m+2 } \sum_{k \in \tilde{S}} \frac{ 1   }{ \Delta_k }\ a.s.
\end{equation}
\end{clm}
\begin{proof}
Fix $l \in \{1,\ldots,\ceil{ \log_{\eta}(1/\gamma) } \}$ and define the set
\begin{equation}
\underline{S}^{(i)}(l) = \underline{S}^{(i)} \cap \left( \cup_{j = \ceil{ T^{ \gamma \eta^{l-1} / \beta } } \vee \tau }^{ \ceil{ T^{ \gamma \eta^l / \beta } } } S_{j}^{(i)} \right) .
\end{equation}
Note it suffices to show $|\underline{S}^{(i)}(l)| \leq m+2\ a.s.$; indeed, if this inequality holds, we obtain
\begin{equation}
\sum_{k \in \underline{S}^{(i)}}  \frac{ 1   }{ \Delta_k } 1 \left( k \in \cup_{j = \ceil{ T^{ \gamma \eta^{l-1} / \beta } } \vee \tau }^{ \ceil{ T^{ \gamma \eta^l / \beta } } } S_{j}^{(i)} \right) = \sum_{k \in \underline{S}^{(i)}(l)} \frac{1}{\Delta_k} \leq \max_{ \tilde{S} \subset \underline{S}^{(i)} : |\tilde{S}| \leq m+2 } \sum_{k \in \tilde{S}} \frac{ 1   }{ \Delta_k }\ a.s. 
\end{equation}

To prove $|\underline{S}^{(i)}(l)| \leq m+2$, we show $|\underline{S}^{(i)}(l)| > m+2$ yields a contradiction. If $|\underline{S}^{(i)} |\leq m+2$, we are done, so we assume $\underline{S}^{(i)} \geq m+3$. For this nontrivial case, we begin with some definitions. First, let $k_1, \ldots , k_{m+3}$ be distinct elements of $\underline{S}^{(i)}(l)$. For $b \in \{1,\ldots,m+3\}$, set
\begin{equation}\label{eqLateNonStickyRecDefn}
j_b = \min \left\{ j \in \{ \ceil{ T^{ \gamma \eta^{l-1} / \beta } } \vee \tau , \ldots , \ceil{ T^{ \gamma \eta^l / \beta } } \} : k_b \in S_j^{(i)} \right\} .
\end{equation}
Note $j_b$ is well-defined since $k_b \in \underline{S}^{(i)}(l)$. Also note we can assume (without loss of generality, after possibly relabeling $\{ k_b \}_{b=1}^{m+3}$) that $j_1 \leq \cdots \leq j_{m+3}$. We claim
\begin{equation}\label{eqLateNonStickyRecRes}
j_b > \ceil{ T^{ \gamma \eta^{l-1} / \beta } } \vee \tau\ \forall\ b \in \{3,\ldots,m+3\} .
\end{equation}
This is easily proven by contradiction. Suppose $j_b = \ceil{ T^{ \gamma \eta^{l-1} / \beta } } \vee \tau$ for some $b \geq 3$. Then since $j_1 \leq \cdots \leq j_b$ by assumption and $j_{b'} \geq \ceil{ T^{ \gamma \eta^{l-1} / \beta } } \vee \tau\ \forall\ b'$ by definition, we must have $j_1 = \cdots = j_b$. Consequently, $k_1, \ldots , k_b \in S_{j_1}^{(i)}$, which implies $| S_{j_1}^{(i)} \cap \underline{S}^{(i)} | \geq | \{ k_1, \ldots , k_{b} \} |\geq 3$; in words, $S_{j_1}^{(i)}$ contains three non-sticky arms. But $S_{j_1}^{(i)}$ contains exactly two non-sticky arms in Algorithm \ref{algGeneral}, so we have a contradiction.

Having established \eqref{eqLateNonStickyRecRes}, and using the definition \eqref{eqLateNonStickyRecDefn}, we conclude $k_b \in S_{j_b}^{(i)} \setminus S_{j_b-1}^{(i)}\ \forall\ b \in \{3,\ldots,m+3\}$, i.e., $k_b$ was not active at phase $j_b-1$ but became active at phase $j_b$. Also note $|S_{j_b}^{(i)} \setminus S_{j_b-1}^{(i)}| \leq 1$ in Algorithm \ref{algGeneral}, i.e., at most one arm is newly-active at each phase. Combined with the fact that $\{ k_b \}_{b=3}^{m+3}$ are distinct arms, $\{ j_b \}_{b=3}^{m+3}$ must be distinct phases. Therefore,
\begin{equation}\label{eqLateNonStickyRecResStrong}
\ceil{ T^{ \gamma \eta^{l-1} / \beta } } \vee \tau < j_3 < j_4 < \cdots < j_{m+3} \leq \ceil{ T^{ \gamma \eta^{l} / \beta } } .
\end{equation}

Next, note $k_b \in S_{j_b}^{(i)} \setminus S_{j_b-1}^{(i)}$ implies $k_b = R_{j_b-1}^{(i)}$ (i.e., $k_b$ was recommended at phase $j_b-1$). Further, $k_b \in \underline{S}^{(i)}$ implies $k_b \neq 1$ (since $\underline{S}^{(i)}$ is a subset of suboptimal arms) and $j_b > \tau$ implies $j_b-1 \geq \tau$ (since $j_b \in \N$); taken together, we must have $H_{j_b-1}^{(i)} \in \{n+1,\ldots,n+m\}$ (i.e., the  arm $k_b$ was recommended by a malicious agent, which follows from \eqref{eqBestAfterTau}). Since $\{ j_b \}_{b=3}^{m+3}$ and $\{n+1,\ldots,n+m\}$ contain $m+1$ and $m$ elements, respectively, the pigeonhole principle says that for some  $i^* \in \{n+1,\ldots,n+m\}$, $b , b' \in \{3,\ldots,m+3\}$ such that $b \neq b'$, 
$H_{j_b-1}^{(i)} = H_{j_{b'}-1}^{(i)} = i^*$. Assume (without loss of generality) that $b < b'$. Recall $R_{j_b-1}^{(i)} \neq 1$; also, since $j_b > \tau$, \eqref{eqBestAfterTau} implies $B_{j_b}^{(i)} = 1$; thus, $B_{j_b}^{(i)} \neq R_{j_b-1}^{(i)}$. It follows from Algorithm \ref{algUpdateGossip2} that $i^* \in P_j^{(i)}\ \forall\ j \in \{ j_b , \ldots \ceil{ j_b^{\eta} } \}$, i.e., malicious agent $i^*$ was blocked until phase $\ceil{ j_b^{\eta} }$. But by \eqref{eqLateNonStickyRecResStrong} and the fact that $\eta > 1$,
\begin{equation}
\ceil{ j_b^{\eta} } \geq j_b^{\eta} \geq ( T^{\gamma \eta^{l-1} / \beta} + 1 )^{\eta} \geq T^{\gamma \eta^l / \beta }  + 1 \geq \ceil{ T^{\gamma \eta^l / \beta } } \geq j_{b'} > j_{b'}-1 ,
\end{equation}
so that $i^* \in P_{j_{b'}-1}^{(i)}$, contradicting $H_{j_{b'}-1}^{(i)} = i^*$.
\end{proof}

Using these claims, we derive an almost-sure bound for the sum of random variables in \eqref{eqLateNonStickyAsTerm}:
\begin{align}
\sum_{l=1}^{ \ceil{ \log_{\eta}(1/\gamma) } }  \sum_{k \in \underline{S}^{(i)}} \Delta_k  \sum_{t = 1 + A_{ \ceil{ T^{ \gamma \eta^{l-1} / \beta } } } }^{ A_{ \ceil{ T^{ \gamma \eta^l / \beta }  } }  } Y_{k,l,t}
& =  \sum_{l=1}^{ \ceil{ \log_{\eta}(1/\gamma) } }  \sum_{k \in \underline{S}^{(i)}}  1 \left( k \in \cup_{j = \ceil{ T^{ \gamma \eta^{l-1} / \beta } } \vee \tau }^{ \ceil{ T^{ \gamma \eta^l / \beta } } } S_{j}^{(i)} \right)  \left( \Delta_k  \sum_{t = 1 + A_{ \ceil{ T^{ \gamma \eta^{l-1} / \beta } } } }^{ A_{ \ceil{ T^{ \gamma \eta^l / \beta }  } } } Y_{k,l,t} \right) \label{eqClmLateNonStickyIndicatorApp} \\
&  \leq  \sum_{l=1}^{ \ceil{ \log_{\eta}(1/\gamma) } }  \sum_{k \in \underline{S}^{(i)}}  1 \left( k \in \cup_{j = \ceil{ T^{ \gamma \eta^{l-1} / \beta } } \vee \tau }^{ \ceil{ T^{ \gamma \eta^l / \beta } } } S_{j}^{(i)} \right)  \frac{ 4 \alpha \log ( A_{ \ceil{ T^{ \gamma \eta^l / \beta }  } } \wedge T ) }{ \Delta_k } \\
& \leq 4 \alpha \left( \max_{ \tilde{S} \subset \underline{S}^{(i)} : |\tilde{S}| \leq m+2 } \sum_{k \in \tilde{S}} \frac{ 1   }{ \Delta_k }  \right) \sum_{l=1}^{ \ceil{ \log_{\eta}(1/\gamma) } } \log ( A_{ \ceil{ T^{ \gamma \eta^l / \beta }  }} \wedge T ) \label{eqClmLateNonStickyIndicatorSumApp}  \\
&  \leq \frac{2 \eta-1}{\eta-1} \max_{ \tilde{S} \subset \underline{S}^{(i)} : |\tilde{S}| \leq m+2 } \sum_{k \in \tilde{S}} \frac{ 4 \alpha \log T }{ \Delta_k } +  \frac{ 8 \alpha \beta \log_{\eta}(1/\gamma) (m+2) }{\Delta_2} . \label{eqLateNonStickyAsBound} 
\end{align}
Here the first equality uses Claim \ref{clmLateNonStickyIndicator}, the first inequality holds by the argument of \eqref{eqLateStickyAsBound} in the proof of Lemma \ref{lemLateSticky}, the second uses Claim \ref{clmLateNonStickyIndicatorSum}, and the third uses Claim \ref{clmLateNonStickyLogSum} from Appendix \ref{appAuxiliary} and $\Delta_2 \leq \Delta_k$. The proof of the lemma is completed by substituting \eqref{eqLateNonStickyConcBound} into \eqref{eqLateNonStickyConcTerm} and \eqref{eqLateNonStickyAsBound} into \eqref{eqLateNonStickyAsTerm}.

\subsection{Early regret (proof of Lemma \ref{lemEarly})} \label{appEarly}

We begin with a simple identity: for any $j' \in \N$,
\begin{equation}
A_{j'} = \sum_{j=1}^{j'} ( A_j - A_{j-1}  ) = \sum_{j=1}^{\infty}  ( A_j - A_{j-1}  ) 1 ( j \leq j' ) . 
\end{equation}
Using this identity and rearranging summations yields
\begin{equation}
\E A_{\tau} = \sum_{j'=1}^{\infty} A_{j'} \P( \tau = j' )    = \sum_{j=1}^{\infty} ( A_j - A_{j-1} ) \P( \tau \geq j ) .
\end{equation}
Now define $f : \N \rightarrow \N$ by $f(j) = \ceil{ 2 + j^{1/\eta} / 2 }\ \forall\ j \in \N$. Then clearly
\begin{equation}
\P( \tau \geq j ) \leq \P  ( \tau_{stab} \geq f(j) ) + \P ( \tau_{stab} < f(j) , \tau \geq j ) .
\end{equation}
Combining the above, we obtain
\begin{equation}\label{eqEarlyTwoTerms}
\E A_{\tau} \leq \sum_{j=1}^{\infty} ( A_j - A_{j-1} ) \P  ( \tau_{stab} \geq f(j) ) + \sum_{j=1}^{\infty} ( A_j - A_{j-1} ) \P ( \tau_{stab} < f(j) , \tau \geq j )  .
\end{equation}

While the definition of $f$ is somewhat opaque, the key property is that $f(j) = \Theta( j^{1/\eta} )$ (the constants are chosen for analytical convenience). This property ensures that if $\tau_{stab} < f(j)$, any blocking that occurred before phase $\tau_{stab}$ ends by $\tau_{stab}^{\eta} < f(j)^{\eta} = \Theta(j)$. In particular, any honest $i^*$ with $1 \in \hat{S}^{(i^*)}$ will \textit{not} be blocked at phase $\Theta(j)$ (since $i^*$ only recommends arm $1$, and $i$ subsequently pulls this arm most frequently, after $\tau_{stab}$). This idea allows us to bound the second term in \eqref{eqEarlyTwoTerms}. The first term in \eqref{eqEarlyTwoTerms} can be bounded using tail bounds for $\tau_{stab}$ from \cite{chawla2020gossiping}.

\begin{clm} \label{clmBestArmIdentify}
Under the assumptions of Lemma \ref{lemEarly},
\begin{align}
\sum_{j=1}^{\infty} ( A_j - A_{j-1} ) \P  ( \tau_{stab} \geq f(j) ) & \leq 2^{1 + \beta \eta} \left( 4 + \left(\frac{26 \alpha (S+2)}{(\beta-1) \Delta_2^2}\right)^{2/(\beta-1)} \right) ^{\beta \eta} + \frac{ 2^{ \beta(2\alpha-3) +1 } n \binom{K}{2} (S+1)  }{ (2\alpha-3) (\beta(2\alpha-3) - 1) ( ( \beta (2\alpha-3)-1)/\eta - \beta )  } .
\end{align}
\end{clm}
\begin{proof}
We first use ideas from \cite{chawla2020gossiping} to derive a tail bound for $\tau_{stab}$. To begin, let
\begin{equation}\label{eqDefnJ1star}
j_1^* = \min \left\{ j \in \N : \frac{ A_{j'} - A_{j'-1} }{ S + 2 } \geq 1 + \frac{4 \alpha \log A_{j'} }{ \Delta_2^2 }\ \forall\ j' \in \{ f(j) , f(j)+1, \ldots\} \right\} ,
\end{equation}
and fix $j \geq j_1^*$. Note that by definition of $\tau_{stab}$ and the union bound,
\begin{equation}\label{eqEarlyTauStabUnion}
\P ( \tau_{stab} \geq f(j) ) \leq \sum_{i=1}^n \sum_{j' = f(j) }^{\infty} \P ( \chi_{j'}^{(i)} = 1 ) .
\end{equation}
Now since $j \geq j_1^*$ and $f$ is increasing, $f(j) \geq f(j_1^*)$, so by definition, any $j' \geq f(j)$ satisfies
\begin{equation}
\frac{ A_{j'} - A_{j'-1} }{ S + 2 } \geq 1 + \frac{4 \alpha \log A_{j'} }{ \Delta_2^2 } .
\end{equation}
This is the assumption of \cite[Lemma 8]{chawla2020gossiping}, so we can apply this lemma to obtain
\begin{equation}
\P ( \chi_{j'}^{(i)} = 1 ) \leq \frac{2 \binom{K}{2} ( S + 1 ) }{ 2 \alpha -3 } A_{j'-1}^{ - (2\alpha -3 ) } .
\end{equation}
(Note $\beta, \eta > 1, 1+\beta \eta < \beta ( 2\alpha - 3)$ ensures $2 \alpha -3 > 0$. Also, $\ceil{K/n}$ appears in \cite[Lemma 8]{chawla2020gossiping} instead of $S$, because \cite{chawla2020gossiping} assumes $S = \ceil{K/n}$; however, the proof follows for general $S$.) Thus,
\begin{align}\label{eqEarlySingleAgent}
\sum_{j' = f(j) }^{\infty} \P ( \chi_{j'}^{(i)} = 1 ) & \leq \frac{2 \binom{K}{2} ( S + 1 ) }{ 2 \alpha -3 } \sum_{j' = f(j) }^{\infty} A_{j'-1}^{ - (2\alpha -3 ) }
\end{align}
We estimate the summation on the right side with an integral as follows:
\begin{align} 
\sum_{j' = f(j) }^{\infty} A_{j'-1}^{ - (2\alpha -3 ) }  & = \sum_{j' = f(j) }^{\infty} \ceil{ (j'-1)^{\beta} }^{ - (2\alpha-3) }  \leq \sum_{j' = f(j) }^{\infty} (j'-1)^{ - \beta (2\alpha-3) }  \\
& \leq \int_{j' = f(j)}^{\infty} (j'-2)^{ - \beta (2\alpha-3) } dj'  = \frac{( f(j) - 2 )^{ 1 - \beta(2\alpha-3)}}{ \beta(2\alpha-3) - 1 }  \label{eqEarlyIntegralApproxFj} \leq \frac{ 2^{ \beta(2\alpha-3) - 1 } j^{ (1-\beta(2\alpha-3))/\eta} }{ \beta(2\alpha-3) - 1 } ,
\end{align}
where the final inequality is by definition of $f(j)$ (note $1+\beta \eta < \beta ( 2\alpha - 3)$ guarantees $\beta(2\alpha-3)-1 > 0$.) Together with \eqref{eqEarlyTauStabUnion} and \eqref{eqEarlySingleAgent}, we have shown
\begin{equation}
\P ( \tau_{stab} \geq f(j) ) \leq \frac{ 2^{ \beta(2\alpha-3)   } n \binom{K}{2} (S+1)  j^{ (1-\beta(2\alpha-3))/\eta} }{ (2\alpha-3) (\beta(2\alpha-3) - 1) }\ \forall\ j \geq j_1^*.
\end{equation}

Using this tail bound, we bound the quantity of interest. First, we note that since $A_j = \ceil{ j^{\beta} }$, the mean value theorem guarantees that for any $j \in \N$ and some $\tilde{j} \in (j-1,j)$,
\begin{equation}\label{eqAjMVT}
A_j - A_{j-1} \leq j^{\beta}  - (j-1)^{\beta} + 1 = \beta \tilde{j}^{\beta-1} + 1 \leq \beta j^{\beta-1} + 1 \leq 2 \beta j^{\beta-1} .
\end{equation}
Combining the previous two inequalities, we thus obtain
\begin{align}\label{eqBestArmSpreadUseMVT}
\sum_{j=1}^{\infty} ( A_j - A_{j-1} ) \P  ( \tau_{stab} \geq f(j) ) & \leq A_{j_1^*} + \frac{ 2^{ \beta(2\alpha-3)  +1 } \beta n \binom{K}{2} (S+1) }{ (2\alpha-3) (\beta(2\alpha-3) - 1) } \sum_{j=j_1^*+1}^{\infty} j^{-1 + \beta + (1-\beta(2\alpha-3))/\eta} \\
&  \leq A_{j_1^*} + \frac{ 2^{ \beta(2\alpha-3) +1 } n \binom{K}{2} (S+1)  }{ (2\alpha-3) (\beta(2\alpha-3) - 1) ( ( \beta (2\alpha-3)-1)/\eta - \beta )  } ,
\end{align}
where the second inequality holds by an integral approximation like \eqref{eqEarlyIntegralApproxFj} and uses $1+\beta \eta < \beta ( 2\alpha - 3)$. Now using the definition of $A_j$ and applying Claim \ref{clmBoundJ1star} from Appendix \ref{appAuxiliary} with $\lambda = 1$, we have
\begin{equation} \label{eqApplyJ1star}
A_{j_1^*} \leq ( j_1^* )^{\beta} + 1 \leq 2 ( j_1^* )^{\beta} \leq 2^{1 + \beta \eta} \left( 4 + \left(\frac{26 \alpha (S+2)}{(\beta-1) \Delta_2^2}\right)^{2/(\beta-1)} \right) ^{\beta \eta} .
\end{equation}
Combining the previous two inequalities completes the proof.
\end{proof}

\begin{rem} \label{remImproveDelta}
The $\Delta_2^{ -4 \beta \eta / (\beta-1) }$ scaling of our regret bound arises from \eqref{eqApplyJ1star}. For any \emph{fixed} $\varepsilon > 0$, this can be improved to $\Delta_2^{-2 (1+\varepsilon)^2 \beta / (\beta-1)}$ by setting $\eta = 1+\varepsilon$ in the algorithm and choosing $\lambda = 1/\varepsilon$ (instead of $\lambda = 1$) when applying Claim \ref{clmBoundJ1star}. However, choosing $\lambda = 1/\varepsilon$ inflates the constant $26$ to $13 ( 1 + 1/\varepsilon )$, so this only works for fixed $\varepsilon$. Owing to this, and to simplify our ultimate regret bound, we simply choose $\lambda = 1$.
\end{rem}

\begin{clm} \label{clmBestArmSpread}
Under the assumptions of Lemma \ref{lemEarly},
\begin{align}
\sum_{j=1}^{\infty} ( A_j - A_{j-1} ) \P ( \tau_{stab} < f(j) , \tau \geq j ) \leq \frac{10 \beta}{ \beta-1} \max \{ 6 (m+n) \max \{ \log n , 2 (\beta-1) \} ,  3 ( 6^{\eta} + 2 ) \}^{\beta} .
\end{align}
\end{clm}
\begin{proof}
We begin by bounding the probability terms for large $j$. In particular, we define
\begin{align}\label{eqDefnJ2star}
j_2^* = \min \{ & j \in \N \cap [ \max \{ 8 , 6 (m+n) \max \{ \log n , 2 (\beta-1) \} \} , \infty ) : j' \geq 3 \ceil{ f(j')^{\eta} } / 2\ \forall\ j' \in \{j,j+1,\ldots\} \} ,
\end{align}
and we derive a bound $j \in \{ j_2^* + 1 , j_2^* + 2 , \ldots \}$. We first note that by definition of $f$ and since $\eta > 1$,
\begin{equation}\label{eqJandFJaboveJ2star}
j \geq \frac{3 \ceil{ f(j)^{\eta} } }{2} \geq \frac{3 f(j)}{2}  = f(j) + \frac{f(j)}{2} \geq f(j) + \frac{ \ceil{ 5/2 } }{2} = f(j) + \frac{3}{2} .
\end{equation}
Now to bound the probability terms, we first use the definition of $\tau$ and the union bound to write
\begin{equation}\label{eqBestArmSpreadUnion}
\P ( \tau_{stab} < f(j) , \tau \geq j ) \leq \sum_{i =1}^n \P ( \tau_{stab} < f(j)  ,  \inf \{ j' \geq \tau_{stab} : 1 \in S_{j'}^{(i)} \} > j-1 ) .
\end{equation}
We fix $i \in \{1,\ldots,n\}$ and bound the $i$-th summand in \eqref{eqBestArmSpreadUnion}. We first observe
\begin{equation}\label{eqBestArmSpreadFirstImp}
\tau_{stab} < f(j)  ,  \inf \{ j' \geq \tau_{stab} : 1 \in S_{j'}^{(i)} \} > j-1 \quad \Rightarrow \quad  \tau_{stab} < f(j) , 1 \notin S_{j-1}^{(i)} ,
\end{equation}
which is easily proven by contradiction: if the left side of \eqref{eqBestArmSpreadFirstImp} holds but $1 \in S_{j-1}^{(i)}$, \eqref{eqJandFJaboveJ2star} ensures $j - 1 > f(j) > \tau_{stab}$, so $j-1 \in \{ j' \geq \tau_{stab} : 1 \in S_{j'}^{(i)} \}$, contradicting the left side of \eqref{eqBestArmSpreadFirstImp}. From \eqref{eqBestArmSpreadFirstImp}, we immediately see the $i$-th summand in \eqref{eqBestArmSpreadUnion} is zero if $1 \in \hat{S}^{(i)}$. In the nontrivial case $1 \notin \hat{S}^{(i)}$, we let $i^*$ be any agent with $1 \in \hat{S}^{(i^*)}$ (such an agent exists by assumption) and claim
\begin{equation}\label{eqBestArmSpreadSecondImp}
\tau_{stab} < f(j) , 1 \notin S_{j-1}^{(i)} \quad \Rightarrow \quad \tau_{stab} < f(j) , H_{j'}^{(i)} \neq i^*\ \forall\ j' \in \{ f(j) -1 , \ldots , j-2 \} .
\end{equation}
Suppose instead that $H_{j'}^{(i)} = i^*$ for some $j' \in \{ f(j) -1 , \ldots , j-2 \}$ (note the set is nonempty by \eqref{eqJandFJaboveJ2star}). Then since $j' \geq f(j)-1 \geq \tau_{stab}$, the definition of $\tau_{stab}$ ensures $R_{j'}^{(i)} = 1$ ($i^*$ only recommends $1$ at and after $\tau_{stab}$), so $1 \in S_{j'+1}^{(i)}$ by Algorithm \ref{algGeneral}. If $j' = j-2$, this  contradicts $1 \notin S_{j-1}^{(i)} = S_{j'+1}^{(i)}$. If $j' < j-2$, the definition of $\tau_{stab}$ yields the same contradiction (since $i$ never discards the best arm after $\tau_{stab}$). This completes the proof of \eqref{eqBestArmSpreadSecondImp}. However, it will be more convenient to use a weaker version (which follows from \eqref{eqBestArmSpreadSecondImp} since $\ceil{ f(j)^{\eta} } \geq f(j)$):
\begin{equation}\label{eqBestArmSpreadSecondImpWeak}
\tau_{stab} < f(j) , 1 \notin S_{j-1}^{(i)} \quad \Rightarrow \quad \tau_{stab} < f(j) , H_{j'}^{(i)} \neq i^*\ \forall\ j' \in \{ \ceil{ f(j)^{\eta} } + 1 , \ldots , j-2 \} ,
\end{equation}
(Note $(j-2) - (\ceil{ f(j)^{\eta} } + 1) \geq j/3 - 3 \geq 0$ since $j \geq j_2^* +1 \geq 9$, so the set in \eqref{eqBestArmSpreadSecondImpWeak} is nonempty.) Finally, we derive one further implication:
\begin{equation}\label{eqBestArmSpreadThirdImp}
\tau_{stab} < f(j) \quad \Rightarrow \quad i^* \notin P_{j'}^{(i)}\ \forall\ j' \in \{ \ceil{ f(j)^{\eta} } + 1 ,  \ceil{ f(j)^{\eta} } + 2 , \ldots \} .
\end{equation}
To prove \eqref{eqBestArmSpreadThirdImp}, we define $j^* = \sup \{ j' \in \{2,3,\ldots\} : i^* \in P_{j'}^{(i)} \setminus P_{j'-1}^{(i)} \}$ to be the latest phase at which $i^*$ entered the blocklist. We consider two cases:
\begin{itemize}
\item $j^* > f(j)$: First note $i^* \in P_{j^*}^{(i)} \setminus P_{j^*-1}^{(i)}$ implies $H_{j^*-1}^{(i)} = i^*$ and $B_{j^*}^{(i)} \neq R_{j^*-1}^{(i)}$ (i.e., to enter the blocklist at $j^*$, $i^*$ must recommend an arm to $i$ at $j^*-1$ that was not $i$'s most played in phase $j^*$.) Since $f(j) > \tau_{stab}$ and $j^*, f(j), \tau_{stab} \in \N$, we must have $j_* \geq \tau_{stab}+2$, so $j^*-1 > \tau_{stab}$, and $B_{j^*}^{(i)} = R_{j^*-1}^{(i)} = 1$ by definition of $\tau_{stab}$. Thus, this case cannot occur.
\item $j^* \leq f(j)$: Suppose the right side of \eqref{eqBestArmSpreadThirdImp} fails, i.e., $i^* \in P_{j'}^{(i)}$ for some $j' \geq \ceil{ f(j)^{\eta} } + 1$. Then by Algorithm \ref{algUpdateGossip2}, there must be some phase $j_{\star}$ such that $i^* \in P_{j_{\star}}^{(i)} \setminus P_{j_{\star}-1}^{(i)}$ and $j' \in \{ j_{\star}, \ldots, \ceil{j_{\star}^{\eta} } \}$ (i.e., $i^*$ entered the blocklist at $j_{\star}$ and $j'$ lies within the blocking period); in particular, $j' \leq \ceil{ j_{\star}^{\eta} }$ But $\ceil{ j_{\star}^{\eta} } \leq \ceil{ (j^*)^{\eta} } \leq \ceil{ f(j)^{\eta} } < j'$ (by definition of $j^*$ and assumption on $f(j) , j'$), a contradiction.
\end{itemize}

Stringing together the implications \eqref{eqBestArmSpreadFirstImp}, \eqref{eqBestArmSpreadSecondImpWeak}, and \eqref{eqBestArmSpreadThirdImp}, we have shown
\begin{equation}\label{eqBestArmSpreadApplyImp}
\tau_{stab} < f(j)  ,  \inf \{ j' \geq \tau_{stab} : 1 \in S_{j'}^{(i)} \} > j-1 \Rightarrow \cap_{j' = \ceil{ f(j)^{\eta} } + 1}^{ j-2} \{ i^* \notin P_{j'}^{(i)} , H_{j'}^{(i)} \neq i^* \}  .
\end{equation} 
We bound the probability of the event at right by writing
\begin{align} 
& \P ( \cap_{j' = \ceil{ f(j)^{\eta} } + 1}^{ j-2} \{ i^* \notin P_{j'}^{(i)} , H_{j'}^{(i)} \neq i^* \} ) \\
& \quad \leq \P ( H_{j-2}^{(i)} \neq i^* | \{ i^* \notin P_{j-2}^{(i)} \} \cap \cap_{j' = \ceil{ f(j)^{\eta} } + 1}^{ j-3} \{ i^* \notin P_{j'}^{(i)} , H_{j'}^{(i)} \neq i^* \} ) \P ( \cap_{j' = \ceil{ f(j)^{\eta} } + 1}^{ j-3} \{ i^* \notin P_{j'}^{(i)} , H_{j'}^{(i)} \neq i^* \} ) \\
& \quad < \left( 1 - \frac{1}{m+n} \right) \P ( \cap_{j' = \ceil{ f(j)^{\eta} } + 1}^{ j-3} \{ i^* \notin P_{j'}^{(i)} , H_{j'}^{(i)} \neq i^* \} ) < \cdots < \left( 1 - \frac{1}{m+n} \right)^{j - \ceil{ f(j)^{\eta} } - 2 } \leq 4 \left( 1 - \frac{1}{m+n} \right)^{j / 3 }  , \label{eqBestArmSpreadGeometric}
\end{align}
where the second inequality holds since $H_{j-2}^{(i)}$ is chosen uniformly from $[m+n] \setminus ( \{ i \} \cap P_{j-2}^{(i)} )$, which (conditioned on $i^* \notin P_{j-2}^{(i)}$) contains at most $m+n-1$ agents, including $i^*$, and the fourth uses $m,n \in \N$ and $j \geq 3 \ceil{ f(j)^{\eta} } / 2$ by definition of $j_2^*$. Combining \eqref{eqBestArmSpreadUnion}, \eqref{eqBestArmSpreadApplyImp}, and \eqref{eqBestArmSpreadGeometric},
\begin{equation}\label{}
\P ( \tau_{stab} < f(j) , \tau \geq j ) \leq 4 n \left( 1 - \frac{1}{m+n} \right)^{j/3} .
\end{equation} 
Finally, we write
\begin{align}
\sum_{j=1}^{\infty} ( A_j - A_{j-1} ) \P ( \tau_{stab} < f(j) , \tau \geq j ) 
&  \leq A_{j_2^*} + 8 \beta \sum_{j=j_2^*+1}^{\infty} j^{\beta-1} n \left( 1 - \frac{1}{m+n} \right)^{j/3}  \leq 2 (j_2^*)^{\beta} + \frac{ 8 \beta ( j_2^* )^{\beta} }{ (\beta-1) n } \leq \frac{10 \beta}{ \beta-1} ( j_2^* )^{\beta} \\
&  \leq \frac{10 \beta}{ \beta-1} \max \{ 6 (m+n) \max \{ \log n , 2 (\beta-1) \} ,  3 ( 6^{\eta} + 2 ) \}^{\beta} ,
\end{align}
where the first inequality follows the argument of \eqref{eqAjMVT}-\eqref{eqBestArmSpreadUseMVT} from the proof of Claim \ref{clmBestArmIdentify}, the second uses Claim \ref{clmBoundJ2starTail} from Appendix \ref{appAuxiliary} and $A_j = \ceil{ j^{\beta} } \leq j^{\beta}+1 \leq 2 j^{\beta}$, the third uses $\beta > 1$ and $n \in \N$, and the fourth uses Claim \ref{clmBoundJ2star} from Appendix \ref{appAuxiliary}.
\end{proof}

\subsection{Auxiliary inequalities} \label{appAuxiliary}

\begin{clm} \label{clmLateNonStickyLogSum}
For any $\gamma \in (0,1)$,
\begin{equation}
\sum_{l=1}^{ \ceil{ \log_{\eta}(1/\gamma) } } \log ( A_{ \ceil{ T^{ \gamma \eta^l / \beta }  }} \wedge T ) \leq \frac{ 2 \eta - 1 }{ \eta - 1 } \log (T) + 2 \beta \log_{\eta}(1/\gamma)  .
\end{equation}
\end{clm}
\begin{proof}
We first recall $T \wedge A_{ \ceil{ T^{ \gamma \eta^{ \ceil{ \log_{\eta}(1/\gamma) } } / \beta }  }} = T$ (see  \eqref{eqLateNonStickyFinalAfterT}), so
\begin{equation}\label{eqLateNonStickyLogSumLastTerm}
\sum_{l=1}^{ \ceil{ \log_{\eta}(1/\gamma) } } \log ( A_{ \ceil{ T^{ \gamma \eta^l / \beta }  }} \wedge T ) \leq \sum_{l=1}^{ \ceil{ \log_{\eta}(1/\gamma) } -1 } \log ( A_{ \ceil{ T^{ \gamma \eta^l / \beta }  }}  ) + \log (T) .
\end{equation}
For the remaining sum, first note $A_{ \ceil{ T^{ \gamma \eta^l / \beta }  }} \leq e^{2 \beta} T^{ \gamma \eta^l }$ by an argument similar to \eqref{eqPhaseBound}. Therefore,
\begin{equation} \label{eqLateNonStickyLogSumSplitSum}
\sum_{l=1}^{ \ceil{ \log_{\eta}(1/\gamma) } -1 } \log ( A_{ \ceil{ T^{ \gamma \eta^l / \beta }  }}  ) \leq \gamma \log(T) \sum_{l=1}^{ \ceil{ \log_{\eta}(1/\gamma) } -1 } \eta^l + 2 \beta \log_{\eta}(1/\gamma) ,
\end{equation}
where we also used $\ceil{x} \leq x+1$. On the other hand, we observe
\begin{equation} \label{eqLateNonStickyLogSum2}
\gamma \log(T) \sum_{l=1}^{ \ceil{ \log_{\eta}(1/\gamma) } -1 } \eta^l = \gamma \log(T) \frac{ \eta^{ \ceil{ \log_{\eta}(1/\gamma) } } - \eta}{ \eta - 1 } \leq \gamma \log(T) \frac{ \frac{\eta}{\gamma} - \eta}{ \eta - 1 } \leq \frac{\eta \log T}{\eta-1} ,
\end{equation}
where the equality computes a geometric series, the first inequality uses $\ceil{x} \leq x+1$, and the second inequality discards a negative term. Combining \eqref{eqLateNonStickyLogSumLastTerm}, \eqref{eqLateNonStickyLogSumSplitSum} and \eqref{eqLateNonStickyLogSum2} completes the proof.
\end{proof}

\begin{clm} \label{clmBoundJ1star}
Assume $\alpha \geq 3/2$ and let $\lambda > 0$. Then $j_1^* \leq j_1$, where $j_1^*$ is defined in \eqref{eqDefnJ1star} and
\begin{equation}
j_1 =  2^{\eta} \left( 4 + \left(\frac{13 (1+\lambda) \alpha (S+2)}{(\beta-1) \Delta_2^2}\right)^{(1+\lambda)/( \lambda (\beta-1)) } \right) ^{\eta} .
\end{equation}
\end{clm}
\begin{proof}
Let $j \in \{ f( j_1 ) , f (j_1)+1,\ldots\}$; by definition of $j_1^*$, we aim show
\begin{equation}
\frac{ A_{j} - A_{j-1} }{ S + 2 } \geq 1 + \frac{4 \alpha \log A_{j} }{ \Delta_2^2 } .
\end{equation}
First recall $f(j_1) \geq j_1^{1/\eta} / 2$ by definition, so 
\begin{equation}\label{eqJ1starJlower}
j \geq f(j_1) \geq \frac{j_1^{1/\eta}}{2}  = 4 + \left(\frac{13 (1+\lambda) \alpha (S+2)}{ (\beta-1) \Delta_2^2}\right)^{(1+\lambda)/(\lambda(\beta-1))} \geq \max \left\{ 4 , \left(\frac{13 (1+\lambda) \alpha (S+2)}{ (\beta-1) \Delta_2^2}\right)^{(1+\lambda)/(\lambda(\beta-1))} \right\} .
\end{equation}
Next, observe that by definition $A_j = \ceil{ j^{\beta}}$, and since $\beta >1$ by assumption and $j \geq 2$ by \eqref{eqJ1starJlower},
\begin{equation}
A_j \leq j^{\beta} + 1 \leq 2 j^{\beta} \leq 2 ( 2(j-1) )^{\beta} = 2^{\beta+1} (j-1)^{\beta} < e^{2 \beta} (j-1)^{\beta} .
\end{equation}
Using this inequality, we can write
\begin{equation}
1 + (S+2) \frac{1 + 4 \alpha \log A_j }{ \Delta_2^2 } \leq 1 + (S+2) \frac{1 + 8 \alpha \beta + 4 \alpha \beta \log (j-1) }{ \Delta_2^2 } .
\end{equation}
Now since $j \geq 4$ by \eqref{eqJ1starJlower}, $\log(j-1) > 1$, so
\begin{equation}
1 + (S+2) \frac{1 + 8 \alpha \beta + 4 \alpha \beta \log (j-1) }{ \Delta_2^2 } < \left( 1 + (S+2) \frac{1 + 12 \alpha \beta  }{ \Delta_2^2 } \right) \log(j-1) .
\end{equation}
For the term in parentheses, we write
\begin{equation}
1 + (S+2) \frac{1 + 12 \alpha \beta}{\Delta_2^2} \leq \frac{ S + 3 + 12 \alpha \beta (S+2) }{\Delta_2^2} < \frac{ 13 \alpha \beta (S+2) }{ \Delta_2^2 } ,
\end{equation}
where the first inequality is $\Delta_2 \leq 1$ and the second is $S+3 \leq \alpha \beta (S+2)$ (which holds since $\alpha \geq 3/2,\beta>1$). Combining the previous three inequalities, we have shown
\begin{align}
1 + (S+2) \frac{1 + 4 \alpha \log A_j }{ \Delta_2^2 } & \leq \frac{ 13 \alpha \beta (S+2) \log(j-1) }{ \Delta_2^2 } = \beta \frac{ 13 (1+\lambda) \alpha (S+2)  }{ (\beta-1) \Delta_2^2 } \log ( (j-1)^{(\beta-1)/(1+\lambda)} ) \\
& \leq \beta (j-1)^{(\beta-1) \lambda / (1+\lambda) } \log ( (j-1)^{(\beta-1) / (1+\lambda) } ) \leq \beta (j-1)^{\beta-1} ,
\end{align}
where the equality rearranges the expression, the second inequality is \eqref{eqJ1starJlower}, and the third inequality is $\log x \leq x$. Rearranging, we have shown
\begin{equation}
\frac{1 + 4 \alpha \log A_j }{ \Delta_2^2 } \leq \frac{ \beta (j-1)^{\beta-1} - 1 }{ S+2 } \leq \frac{ A_j - A_{j-1}}{S+2},
\end{equation}
where the second inequality holds similar to \eqref{eqLowerMVT} in Appendix \ref{appNoBlacklist}.
\end{proof}

\begin{clm} \label{clmBoundJ2starTail}
Defining $j_2^*$ as in \eqref{eqDefnJ2star},
\begin{equation}\label{eqBestArmSpreadTwoTerms}
\sum_{j=j_2^*+1}^{\infty} j^{\beta-1} n \left( 1 - \frac{1}{m+n} \right)^{j/3} \leq \frac{ (j_2^*)^{\beta} }{ (\beta - 1) n } .
\end{equation}
\end{clm}
\begin{proof}
We begin by observing that for any $j \geq j_2^* \geq 6(m+n) \log n$,
\begin{equation}
\left( 1 - \frac{1}{m+n} \right)^{j/6} \leq \exp \left( - \frac{j}{6(m+n)} \right) \leq \frac{1}{n} ,
\end{equation}
where we also used $1-x \leq e^{-x}$. Consequently,
\begin{equation}
n \left( 1 - \frac{1}{m+n} \right)^{j/3} \leq \left( 1 - \frac{1}{m+n} \right)^{j/6} \leq \left( 1 - \frac{1}{ 6(m+n) } \right)^j ,
\end{equation}
where the second inequality is Bernoulli's. Setting $p = 1/(6(m+n))$, it thus suffices to show
\begin{equation}\label{eqJ2starTailToP}
\sum_{j=j_2^*+1}^{\infty} j^{\beta-1} ( 1 - p )^j \leq \frac{ (j_2^*)^{\beta} }{ (\beta - 1) n } .
\end{equation}
Toward this end, first note that whenever $j \geq j_2^*$,
\begin{equation}
\frac{ (j+1)^{\beta-1} (1-p)^{j+1} }{ j^{\beta-1} (1-p)^j } = \left( 1 + \frac{1}{j} \right)^{\beta-1} (1-p) \leq e^{ (\beta-1) / j } (1-p) < e^p (1-p) \leq 1 ,
\end{equation}
where the first and third inequalities are $1+x \leq e^{x}$ and the second uses $(\beta-1) / j \leq p/2 < p$ by definition of $j_2^*$. Thus, the summands in \eqref{eqJ2starTailToP} are decreasing, which implies
\begin{equation}\label{eqEarlyTauIntegralApprox}
\sum_{j=j_2^*+1}^{\infty} j^{\beta-1} (1-p)^j \leq \int_{j=j_2^*}^{\infty} j^{\beta-1} (1-p)^j dj .
\end{equation}
To bound the integral, we write
\begin{align}
\int_{j = j_2^*}^{\infty} j^{\beta-1} (1-p)^j dj & = \frac{1}{\log(1/(1-p))} \left(  (j_2^*)^{\beta-1} ( 1-p )^{j_2^*} + \int_{j = j_2^*}^{\infty} (\beta-1) j^{\beta-2} (1-p)^j dj   \right) \\
& \leq \frac{1}{p} \left(  (j_2^*)^{\beta-1} ( 1-p )^{j_2^*} +  \int_{j = j_2^*}^{\infty} (\beta-1) j^{\beta-2} (1-p)^j dj   \right) = \frac{ (j_2^*)^{\beta-1} ( 1-p )^{j_2^*}  }{p} + \int_{j = j_2^*}^{\infty} \frac{(\beta-1) j^{\beta-2}}{p} (1-p)^j dj ,
\end{align}
where the first equality is obtained via integration by parts, the inequality is $\log (1/x) \geq 1 - x$, and the second equality rearranges the expression. Next, note that by definition of $p$ and $j_2^*$
\begin{equation}
\frac{ (\beta-1) j^{\beta-2} }{ p } = 6(m+n)(\beta-1) j^{\beta-2} = \frac{1}{2} j^{\beta-2} \times 12 (\beta-1) (m+n) \leq \frac{1}{2} j^{\beta-1}\ \forall\ j \geq j_2^* .
\end{equation}
Using the previous two inequalities and rearranging, we obtain
\begin{equation}\label{eqEarlyTauFinalIntegral}
\int_{j = j_2^*}^{\infty} j^{\beta-1} (1-p)^j dj \leq \frac{2 (j_2^*)^{\beta-1} ( 1-p )^{j_2^*}  }{p} \leq \frac{ (j_2^*)^{\beta} }{ (\beta - 1) n } ,
\end{equation}
where the second inequality uses $1 / p = 6(m+n) \leq j_2^* / ( 2 (\beta-1))$ and $(1-p)^{j_2^*} \leq e^{ - p j_2^* } \leq 1/n$, both of which hold by definition of $j_2^*$ and $p$. Plugging into \eqref{eqEarlyTauIntegralApprox} completes the proof.
\end{proof}

\begin{clm} \label{clmBoundJ2star}
$j_2^* \leq j_2$, where $j_2^*$ is defined in \eqref{eqDefnJ2star} and
\begin{equation}
j_2 = \max \{ 6 (m+n) \max \{ \log n , 2 (\beta-1) \} ,  3 ( 6^{\eta} + 2 ) \} .
\end{equation}
\end{clm}
\begin{proof}
By definition of $j_2^*$, showing $j_2^* \leq j_2$ requires us to show
\begin{equation}\label{eqJ2starBoundTwoIneq}
j_2 \geq 6 (m+n) \max \{ \log n , 2 (\beta-1) \} , \quad j_2 \geq 8 , \quad j \geq \frac{3}{2} \ceil{ f(j)^{\eta} }\ \forall\ j \geq j_2 .
\end{equation}
The first inequality is immediate. The second holds since $j_2 > 3 ( 6 + 2 ) = 24$ (since $\eta > 1$). For the third inequality, note that by definition of $f(j)$, $\ceil{x} \leq x+1$, and convexity of $x \mapsto x^{\eta}$, we have
\begin{equation}
\ceil{ f(j)^{\eta} } = \ceil*{ \ceil*{ 2 + \frac{ j^{1/\eta}}{2} }^{\eta} } \leq   1 + \left( 3 + \frac{ j^{1/\eta}}{2} \right)^{\eta}  \leq   1 + \frac{6^{\eta}}{2} + \frac{j}{2}  .
\end{equation}
Therefore, for any $j \geq j_2 \geq 3(6^{\eta}+2)$,
\begin{equation}
\frac{3}{2} \ceil{ f(j)^{\eta} } \leq \frac{ 3 ( 6^{\eta} + 2 ) }{ 4 } + \frac{3j}{4} \leq \frac{j}{4} + \frac{3j}{4} = j ,
\end{equation}
so the third inequality in \eqref{eqJ2starBoundTwoIneq} holds.
\end{proof}

\section{Proof of Theorem \ref{thmNoBlacklist}} \label{appNoBlacklist}

The expected regret bound is a simple consequence of the high probability result. To prove the latter, first define $h(T) = \floor{ ( \ceil{T^{1/\beta}} - 1 ) / 2 }\ \forall\ T \in \N$. Then $h(T) \rightarrow \infty$ as $T \rightarrow \infty$, so for $T$ large,
\begin{equation}
2 h(T) \geq h(T) + 1 \geq \frac{ \ceil{ T^{1/\beta} } - 1  }{ 2 } \geq \frac{ T^{1/\beta} }{ e } , \quad 2 h(T) \leq \ceil{ T^{1/\beta} } - 1 \leq T^{1/\beta} , 
\end{equation}
which respectively imply
\begin{equation}
A_{ 2 h(T) } = \ceil{ ( 2 h(T)  )^{\beta} } \geq ( 2 h(T)  )^{\beta} \geq \frac{T}{e^{\beta}} , \quad A_{2 h(T)} = \ceil{ ( 2 h(T) )^{\beta} } \leq T .
\end{equation}
Consequently, for any $\delta > 0$ and all $T \geq e^{\beta/\delta}$,
\begin{equation}
\frac{ R_T^{(i)} }{ \log T } \geq \frac{ R_{ A_{ 2 h(T) } }^{(i)} }{ \log A_{ 2 h(T) } } \left( 1 - \frac{\beta}{\log T} \right) \geq \frac{ R_{ A_{ 2 h(T) } }^{(i)} }{ \log A_{ 2 h(T) } } (1-\delta) .
\end{equation}
Thus, choosing $\delta$ small enough that $(1-\delta)^2 \geq (1-\varepsilon)$, it suffices to show
\begin{equation}
\lim_{T \rightarrow \infty} \P \left( \frac{R_{A_{2 h(T)}}^{(i)} }{ \log A_{2 h(T)} } < (1-\delta) \alpha \left( 1 - \frac{1}{\sqrt{\alpha}} \right)^2  \sum_{k=2}^K \frac{1}{\Delta_k} \right) = 0 .
\end{equation}
Equivalently (since $h(T) \rightarrow \infty$ as $T \rightarrow \infty$), we can show $\P ( \mathcal{G}_{j_*}^{(i)} ) \rightarrow 0$ as $j_* \rightarrow \infty$, where
\begin{equation}
\mathcal{G}_{j_*}^{(i)} = \left\{ \frac{R_{A_{2 j_*}}^{(i)} }{ \log A_{2 j_*} } < (1-\delta) \alpha \left( 1 - \frac{1}{\sqrt{\alpha}} \right)^2  \sum_{k=2}^K \frac{1}{\Delta_k} \right\} .
\end{equation}
(In words, we have simply rewritten the result in terms of regret at the end of a phase, which will be more convenient.) Thus, our goal is to show $\P ( \mathcal{G}_{j_*}^{(i)} ) \rightarrow 0$. We first eliminate a trivial case where the best arm is not played sufficiently often. Namely, we define the event
\begin{equation}
\mathcal{E}_{j_*}^{(i)} = \{ T_1^{(i)} ( A_{j_*} ) >  j_*^{\beta} / 2 \} \cap \cap_{j=j_*}^{2 j_*} \cup_{t=1+A_{j-1}}^{A_j} \{ I_t^{(i)} = 1 \} ,
\end{equation}
and we show $\mathcal{G}_{j_*}^{(i)} \setminus \mathcal{E}_{j_*}^{(i)} = \emptyset$ large $j_*$ (so it will only remain to show $\P ( \mathcal{G}_{j_*}^{(i)}, \mathcal{E}_{j_*}^{(i)} ) \rightarrow 0$). First note
\begin{equation}
( \mathcal{E}_{j_*}^{(i)} )^C  = \{ T_1^{(i)} ( A_{j_*} ) \leq  j_*^{\beta} / 2 \}  \cup \cup_{j=j_*}^{2 j_*} \{ T_1^{(i)} ( A_j ) = T_1^{(i)} ( A_{j-1} ) \} .
\end{equation}
Now if $T_1^{(i)} ( A_{j_*} ) \leq j_*^{\beta} / 2$, then the number of pulls of suboptimal arms by $A_{2 j_*}$ satisfies
\begin{equation}
\sum_{k=2}^K T_k^{(i)} ( A_{2 j_*} ) \geq \sum_{k=2}^K T_k^{(i)} ( A_{j_*} ) = A_{j_*} - T_1^{(i)} ( A_{j_*} ) = \ceil{ j_*^{\beta} } - T_1^{(i)} ( A_{j_*} )\geq \frac{j_*^{\beta}}{2} \geq \frac{j_*^{\beta-1}}{2} ,
\end{equation}
where the first inequality is monotonicity of $T_k^{(i)}(\cdot)$ and the equalities are by definition. On the other hand, if $T_1^{(i)} ( A_j ) = T_1^{(i)} ( A_{j-1} )$ for some $j \in \{ j_* , \ldots , 2 j_* \}$, then
\begin{equation} \label{eqLowerMVT}
\sum_{k=2}^K T_k^{(i)} ( A_{2 j_*} ) \geq A_j - A_{j-1} \geq j^{\beta} - (j-1)^{\beta} - 1 \geq \beta (j-1)^{\beta-1} - 1 \geq \frac{ j_*^{\beta-1} }{2} ,
\end{equation}
where we again used the definition of $A_j$, along with the mean value theorem, and where the final inequality holds for $j_*$ large. Hence, by the basic regret decomposition  $R_{ A_{2 j_*} }^{(i)} = \sum_{k=2}^K \Delta_k T_k^{(i)}(A_{2 j_*})$,
\begin{equation}
( \mathcal{E}_{j_*}^{(i)} )^C \quad \Rightarrow \quad R_{A_{2 j_*} }^{(i)} \geq \Delta_2 \sum_{k=2}^K T_k^{(i)} ( A_{2 j_*} )  \geq \frac{ \Delta_2 j_*^{\beta-1} }{2} .
\end{equation}
We have shown that $R_{A_{2 j_*} }^{(i)}$ grows polynomially in $j_*$ whenever $ \mathcal{E}_{j_*}^{(i)} $ fails (recall $\beta >1$). On the other hand, $\mathcal{G}_{j_*}^{(i)}$ says $R_{A_{2 j_*} }^{(i)}$ is logarithmic in $j_*$. Thus, $\mathcal{G}_{j_*}^{(i)} \setminus \mathcal{E}_{j_*}^{(i)}$ cannot occur for large $j_*$.

The remainder (and the bulk) of the proof involves showing $\P ( \mathcal{G}_{j_*}^{(i)} , \mathcal{E}_{j_*}^{(i)} ) \rightarrow 0$. We begin with a finite-time lower bound on the number of plays of any suboptimal arm when $\mathcal{E}_{j_*}^{(i)}$ occurs.
\begin{lem} \label{lemFiniteLower}
Let $\alpha >1$, $\beta > 1$, $k \in \{2,\ldots,K\}$, and $j_* \in \N \cap [ ( 2^{\beta} + 1 )^{1/\beta} , \infty)$. Assume that for some $\zeta > 0$ and some $\lambda \in ( \sqrt{ 1/\alpha }, 1 )$,
\begin{gather}
\sqrt{ \alpha \log A_{j_*} } \left( \frac{ \Delta_k ( 1 - \lambda) }{ \sqrt{\zeta \alpha \log A_{2 j_*}} } -  \frac{2 \sqrt{2}}{ j_*^{\beta/2} } \right) \geq \Delta_k \label{eqNonAsympLowerL2} .
\end{gather}
Then for any $i \in \{1,\ldots,n\}$,
\begin{equation} \label{eqFiniteLower}
\P \left( T_k^{(i)}(A_{2 j_*}) < \frac{ \zeta \alpha \log A_{2 j_*} }{ \Delta_k^2 } , \mathcal{E}_{j_*}^{(i)} \right)  \leq \frac{ 2 \zeta \alpha \beta \log ( j_*) j_*^{  2 \beta ( 1 - \alpha \lambda^2 ) } }{  \Delta_k^2 ( \alpha \lambda^2 - 1 ) }   + \left( 1 - \frac{1}{nK} \right)^{j_*} .
\end{equation}
\end{lem}
\begin{proof}
We first use the law of total probability and the union bound to write
\begin{align}
\P \left( T_k^{(i)}(A_{2 j_*}) < \frac{ \zeta \alpha \log A_{2 j_*} }{ \Delta_k^2 } , \mathcal{E}_{j_*}^{(i)} \right) & \leq \sum_{j=j_*+1}^{2 j_*} \P \left( T_k^{(i)}(A_{2 j_*}) < \frac{ \zeta \alpha \log A_{2 j_*} }{ \Delta_k^2 } , \mathcal{E}_{j_*}^{(i)} , k \in S_j^{(i)} \right)\  \label{eqFiniteLowerHardTerm} \\
& \quad\quad + \P ( k \notin S_j^{(i)}\ \forall\ j \in \{j_*+1,\ldots,2 j_*\} ) . \label{eqFiniteLowerEasyTerm}
\end{align}
We will show \eqref{eqFiniteLowerHardTerm} and \eqref{eqFiniteLowerEasyTerm} are bounded by the first and second summands in \eqref{eqFiniteLower}, respectively. We begin with the easier step: bounding \eqref{eqFiniteLowerEasyTerm}. Note \eqref{eqFiniteLowerEasyTerm} is zero for sticky arms $k \in \hat{S}^{(i)}$, so we assume $k \notin \hat{S}^{(i)}$. Then conditioned on $k \notin S_{2 j_*-1}^{(i)}$, $k \in S_{2 j_*}^{(i)} \Leftrightarrow R_{2 j_*-1}^{(i)} = k$ (see Algorithm \ref{algGeneral}). Also, since the malicious agent is contacted with probability $1/n$ at each epoch and recommends uniformly random arms, $R_{2 j_*-1}^{(i)} = k$ with probability at least $1 / ( n K )$. Therefore,
\begin{equation}
\P ( k \in S_{2 j_*}^{(i)} | k \notin S_j^{(i)}\ \forall\ k \in \{j_*+1,\ldots,2 j_*-1\} )  \geq 1 / ( n K ) .
\end{equation}
Subtracting both sides from $1$ and iterating yields the desired bound on \eqref{eqFiniteLowerEasyTerm}:
\begin{align}
\P ( k \notin S_j^{(i)}\ \forall\ k \in \{j_*+1,\ldots,2 j_*\} )  \leq \left( 1 - \frac{1}{nK} \right)^{j_*} .
\end{align}
To bound \eqref{eqFiniteLowerHardTerm}, first let $j \in \{j_*+1,\ldots,2 j_*\}$. Then by definition of $\mathcal{E}_{j_*}^{(i)}$ and the union bound,
\begin{align}
& \P \left( T_k^{(i)}(A_{2 j_*}) < \frac{ \zeta \alpha \log A_{2 j_*} }{ \Delta_k^2 } , \mathcal{E}_{j_*}^{(i)} , k \in S_j^{(i)} \right)  \leq \sum_{t=1+A_{j-1}}^{A_j} \P \left( T_1^{(i)}(A_{j_*}) > \frac{j_*^{\beta}}{2} , T_k^{(i)}(A_{2 j_*}) < \frac{ \zeta \alpha \log A_{2 j_*} }{ \Delta_k^2 } ,  k \in S_j^{(i)}  , I_t^{(i)} = 1  \right) . \label{eqNonAsympLower3beforeUCB}
\end{align}
Next, for $t \in \{1+A_{j-1},\ldots,A_j\}$, we bound the $t$-th summand in \eqref{eqNonAsympLower3beforeUCB} by modifying arguments from \cite{auer2002finite}. First note $k \in S_j^{(i)} ,  I_t^{(i)} = 1$ implies (by Algorithm \ref{algGeneral})
\begin{equation}
\bar{X}_{k, T_k^{(i)} ( t - 1 ) }^{(i)} + c_{t , T_k^{(i)}(t-1) } \leq \bar{X}_{1, T_1^{(i)} ( t - 1 ) }^{(i)} + c_{t , T_1^{(i)}(t-1) } ,
\end{equation}
where $\bar{X}_{k,s}^{(i)}$ is the average of $s$ independent $\text{Bernoulli}(\mu_k)$ random variables and $c_{t,s} = \sqrt{ \alpha \log(t) / s }$. This further implies (by the bounds on $T_1^{(i)}(A_{j_*}) , T_k^{(i)}(A_{2 j_*})$ in \eqref{eqNonAsympLower3beforeUCB}, since $T_1^{(i)}(\cdot), T_k^{(i)}(\cdot)$ are increasing functions, and since $A_{j_*} < 1 + A_{j-1} \leq t \leq A_j \leq A_{2 j_*}$) that
\begin{equation}
\min_{ \sigma_k \in \{1,\ldots, \floor{ \zeta \alpha \log(A_{2 j_*}) / \Delta_k^2} \} } \bar{X}_{k, \sigma_k }^{(i)} + c_{t , \sigma_k } \leq \max_{\sigma_1 \in \{ \ceil{ j_*^{\beta}/2  } ,\ldots,t\}  } \bar{X}_{1, \sigma_1 }^{(i)} + c_{t , \sigma_1 } . 
\end{equation}
Thus, with another union bound, we can bound the $t$-th summand in \eqref{eqNonAsympLower3beforeUCB} by
\begin{align}
& \P \left( T_1^{(i)}(A_{j_*}) > \frac{j_*^{\beta}}{2} , T_k^{(i)}(A_{2 j_*}) < \frac{ \zeta \alpha \log A_{2 j_*} }{ \Delta_k^2 } ,  k \in S_j^{(i)}  , I_t^{(i)} = 1   \right) \leq \sum_{ \sigma_k = 1 }^{\floor{ \zeta \alpha \log(A_{2 j_*}) / \Delta_k^2}  } \sum_{ \sigma_1 = \ceil{ j_*^{\beta}/2} }^t \P ( \bar{X}_{k, \sigma_k }^{(i)} + c_{t , \sigma_k } \leq  \bar{X}_{1, \sigma_1 }^{(i)} + c_{t , \sigma_1 }  ) . \label{eqNonAsympLower3beforeChernoff}
\end{align}
Fixing $\sigma_k, \sigma_1$ as in the double summation, we claim $\bar{X}_{k, \sigma_k }^{(i)} + c_{t , \sigma_k } \leq  \bar{X}_{1, \sigma_1 }^{(i)} + c_{t , \sigma_1 }$ implies that $\bar{X}_{k, \sigma_k }^{(i)} \leq \mu_k - \lambda c_{t, \sigma_k}$ or $\bar{X}_{1, \sigma_1 }^{(i)} \geq \mu_1 + c_{t, \sigma_1}$. Indeed, if both inequalities fail, then
\begin{align}
\bar{X}_{k, \sigma_k }^{(i)} + c_{t , \sigma_k }  & > \mu_k + ( 1 - \lambda) c_{t,\sigma_k} = \mu_1 - \Delta_k + ( 1 - \lambda) c_{t,\sigma_k}
 > \bar{X}_{1, \sigma_1 }^{(i)} - c_{t,\sigma_1} - \Delta_k + ( 1 - \lambda) c_{t,\sigma_k}  \geq \bar{X}_{1, \sigma_1 }^{(i)} + c_{t,\sigma_1} ,
\end{align}
which is a contradiction; here the equality is by definition of $\Delta_k$ and the final inequality holds since
\begin{align}
( 1- \lambda) c_{t,\sigma_k} - 2 c_{t,\sigma_1} & = \sqrt{ \alpha \log t} \left( \frac{1-\lambda}{ \sqrt{\sigma_k} } - \frac{2}{ \sqrt{\sigma_1} } \right) \geq \sqrt{ \alpha \log A_{j_*} } \left( \frac{ \Delta_k (1-\lambda)}{ \sqrt{ \zeta \alpha \log A_{2 j_*}  } } - \frac{2 \sqrt{2}}{ j_*^{\beta/2} } \right) \geq \Delta_k , 
\end{align}
where the first inequality uses $t \geq A_{j_*} , \sigma_k \leq \zeta \alpha \log ( A_{2 l_*} ) / \Delta_k^2, \sigma_1 \geq j_*^{\beta} / 2$ and the second uses \eqref{eqNonAsympLowerL2}. From this implication, we can write
\begin{align}
\P ( \bar{X}_{k, \sigma_k }^{(i)} + c_{t , \sigma_k } \leq  \bar{X}_{1, \sigma_1 }^{(i)} + c_{t , \sigma_1 }  )  & \leq \P ( \bar{X}_{k, \sigma_k }^{(i)} \leq \mu_k - \lambda c_{t, \sigma_k} ) + \P ( \bar{X}_{1, \sigma_1 }^{(i)} \geq \mu_1 + c_{t, \sigma_1} )  e^{ - 2 \alpha \lambda^2 \log t } + e^{- 2 \alpha \log t } < 2 e^{- 2 \alpha \lambda^2 \log t } = 2 t^{-2 \alpha \lambda^2} ,
\end{align}
where the second inequality uses a standard Chernoff bound and the third uses $\lambda < 1$. Combining this inequality with \eqref{eqNonAsympLower3beforeUCB} and \eqref{eqNonAsympLower3beforeChernoff}, then substituting into \eqref{eqFiniteLowerHardTerm}, we have shown
\begin{align}
& \sum_{j=j_*+1}^{2 j_*} \P \left( T_k^{(i)}(A_{2 j_*}) < \frac{ \zeta \alpha \log A_{2 j_*} }{ \Delta_k^2 } , \mathcal{E}_{j_*}^{(i)} , k \in S_j^{(i)} \right) \leq  \sum_{j=j_*+1}^{2 j_*}  \sum_{t=1+A_{j-1}}^{A_j} \sum_{ \sigma_k = 1 }^{\floor{ \zeta \alpha \log(A_{2 j_*}) / \Delta_k^2}  } \sum_{ \sigma_1 = \ceil{ j_*^{\beta} / 2 } }^t 2 t^{-2 \alpha \lambda^2} < \frac{ 2 \zeta \alpha \log A_{2 j_*} }{ \Delta_k^2 }  \sum_{t = 1 + A_{j_*} }^{ \infty} t^{1-2\alpha \lambda^2 } .
\end{align}
We also observe that by assumption $j_* \geq (2^{\beta}+1)^{1/\beta}$, and since $j_*, \beta > 1$,
\begin{equation}
\ceil{ ( 2 j_* )^{\beta} } < ( 2 j_* )^{\beta} + 1 < ( 2^{\beta} +1 ) j_*^{\beta} < j_*^{2\beta} \quad \Rightarrow \quad \log A_{2 j_*} = \log \ceil{ ( 2 j_* )^{\beta} } < 2 \beta \log j_* .
\end{equation}
Finally, we use $A_{j_*} \geq j_*^{\beta}$, $\lambda > \sqrt{ 1 / \alpha }$, and an integral approximation to write
\begin{equation}
\sum_{t = 1 + A_{j_*} }^{ \infty} t^{1-2\alpha \lambda^2 } \leq \int_{t = j_*^{\beta}}^{\infty} t^{1-2\alpha \lambda^2} dt = \frac{ j_*^{  2 \beta ( 1 - \alpha \lambda^2 ) } }{ 2 ( \alpha \lambda^2 - 1 ) } .
\end{equation}
Combining the previous three inequalities yields the desired bound on \eqref{eqFiniteLowerHardTerm}.
\end{proof}

We finish the proof of the theorem by showing $\P ( \mathcal{G}_{j_*}^{(i)} , \mathcal{E}_{j_*}^{(i)} ) \rightarrow 0$. Note by the regret decomposition  $R_{ A_{2 j_*} }^{(i)} = \sum_{k=2}^K \Delta_k T_k^{(i)}(A_{2 j_*})$ and the union bound, it suffices to show that for any $k$,
\begin{equation}\label{eqAsympLower}
\lim_{j_* \rightarrow \infty} \P \left( \frac{T_k^{(i)}(A_{2 j_*})}{ \log A_{2 j_*} } < \frac{ (1-\delta) \alpha (1-1/\sqrt{\alpha})^2 }{ \Delta_k^2 } , \mathcal{E}_{j_*}^{(i)} \right)  = 0 .
\end{equation}
We prove \eqref{eqAsympLower} using Lemma \ref{lemFiniteLower}. First, we define $\lambda = \lambda(j_*)$ by
\begin{equation}\label{eqDefnLambdaJstar}
\lambda(j_*) =  \sqrt{ \frac{1 + 1 / \sqrt{ \log j_* }}{\alpha} } .
\end{equation}
We choose $\zeta = \zeta(j_*)$ such that \eqref{eqNonAsympLowerL2} holds with equality, i.e.,
\begin{equation}\label{eqImplicitDefnZeta}
\sqrt{ \alpha \log A_{j_*} } \left( \frac{ \Delta_k ( 1 - \lambda (j_*) ) }{ \sqrt{\zeta (j_*) \alpha \log A_{2 j_*}} } -  \frac{2 \sqrt{2}}{ j_*^{\beta/2} } \right) = \Delta_k .
\end{equation}
We claim (and will return to prove) $\zeta (j_*) \rightarrow ( 1 - 1/\sqrt{\alpha} )^2$ as $j_* \rightarrow \infty$. Assuming this holds, we have $\zeta(j_*) > (1-\delta ) ( 1 - 1 / \sqrt{\alpha} )^2 > 0$ for all large $j_*$. Also, it is clear that $1 / \sqrt{\alpha} < \lambda(j_*) < 1$ for large $j_*$. Hence, the assumptions of Lemma \ref{lemFiniteLower} hold for large $j_*$, so for such $j_*$,
\begin{equation}\label{eqApplyLowerFinite}
\P \left( \frac{T_k^{(i)}(A_{2 j_*})}{ \log A_{2 j_*} } < \frac{ \alpha \zeta(j_*)   }{ \Delta_k^2 } , \mathcal{E}_{j_*}^{(i)} \right)  \leq \frac{ 2 \alpha \beta \zeta(j_*)  \log ( j_*) j_*^{  2 \beta ( 1 - \alpha \lambda(j_*)^2 ) } }{  \Delta_k^2 ( \alpha \lambda(j_*)^2 - 1 ) }   + \left( 1 - \frac{1}{nK} \right)^{j_*} .
\end{equation}
Note that by monotonicity and $\zeta(j_*) > (1-\delta ) ( 1 - 1 / \sqrt{\alpha} )^2$ for large $j_*$, \eqref{eqAsympLower} will follow if we can show the right side of \eqref{eqApplyLowerFinite} vanishes. Clearly $(1 - 1/ (nK) )^{j_*} \rightarrow 0$. For the first term in \eqref{eqApplyLowerFinite}, note
\begin{align}
\frac{  \log ( j_*) j_*^{  2 \beta ( 1 - \alpha \lambda(j_*)^2 ) } }{   \alpha \lambda(j_*)^2 - 1  } = ( \log j_* )^{3/2} j_*^{ -2 \beta / \sqrt{\log j_*} } = e^{ \frac{3}{2} \log ( \log j_* ) - 2 \beta \sqrt{ \log j_* } } \xrightarrow[j_* \rightarrow \infty]{} 0 ,
\end{align}
so since $\alpha,\beta,\Delta_k$ are constants and $\lim_{j_* \rightarrow \infty} \zeta(j_*) < \infty$, the first term in \eqref{eqApplyLowerFinite} vanishes as well.

It remains to show $\zeta (j_*) \rightarrow ( 1 - 1/\sqrt{\alpha} )^2$. By definition $A_{j_*} = \ceil{j_*^{\beta}}$, one can verify 
\begin{equation}
\lim_{j_* \rightarrow \infty} \frac{ \log A_{j_*} }{ \log A_{2 j_*} } = 1 , \quad \lim_{j_* \rightarrow \infty} \frac{ \log A_{j_*} }{ j_*^{\beta} } = 0 ,
\end{equation}
which, combined with \eqref{eqDefnLambdaJstar} and \eqref{eqImplicitDefnZeta}, implies
\begin{equation}
1 = \lim_{j_* \rightarrow \infty} \frac{1-\lambda(j_*)}{ \sqrt{ \zeta(j_*) } } = \lim_{j_* \rightarrow \infty} \frac{1 - \sqrt{ \frac{1 + 1 / \sqrt{ \log j_* }}{\alpha} }}{ \sqrt{ \zeta(j_*)} } ,
\end{equation}
so $\zeta (j_*) \rightarrow ( 1 - 1/\sqrt{\alpha} )^2$ indeed holds.

\section{Experimental details} \label{appExperimentDetails}

In Table \ref{tabOtherResults}, we show the average regret $\frac{1}{n} \sum_{i=1}^n R_T^{(i)}$ (reported as mean $\pm$ standard deviation) at horizon $T = 10^5$ relative to the algorithm from \cite{chawla2020gossiping} for various values of $m$, $K$, and $\eta$. We use the same synthetic and real datasets, define the same uniform and omniscient malicious agent strategies, and choose $n = 25, \beta = 2 , \alpha = 4, S = \ceil{K/n}$ as in Section \ref{secExperiments}.

\begin{table}
\caption{Average regret at $T = 10^5$ relative to the algorithm from \cite{chawla2020gossiping}} \label{tabOtherResults}
\centering
\begin{tabular}{c c c c c c c c}
\multicolumn{3}{c}{} & \multicolumn{2}{c}{Synthetic data} & \multicolumn{2}{c}{Real data}  \\ 
$m$ & $K$ & $\eta$ & Uniform & Omniscient & Uniform & Omniscient \\ \hline
10 & 75 & 2 & $0.450 \pm 0.160$ & $0.413 \pm 0.052$ & $0.582 \pm 0.157$ & $0.543 \pm 0.136$ & $$ \\ 
10 & 75 & 3 & $0.415 \pm 0.172$ & $0.415 \pm 0.046$ & $0.578 \pm 0.168$ & $0.567 \pm 0.204$ & $$ \\ 
10 & 75 & 4 & $0.435 \pm 0.124$ & $0.395 \pm 0.034$ & $0.525 \pm 0.143$ & $0.593 \pm 0.227$ & $$ \\ 
10 & 100 & 2 & $0.413 \pm 0.086$ & $0.401 \pm 0.076$ & $0.560 \pm 0.142$ & $0.483 \pm 0.080$ & $$ \\ 
10 & 100 & 3 & $0.464 \pm 0.235$ & $0.412 \pm 0.114$ & $0.564 \pm 0.143$ & $0.504 \pm 0.108$ & $$ \\ 
10 & 100 & 4 & $0.418 \pm 0.107$ & $0.404 \pm 0.070$ & $0.535 \pm 0.139$ & $0.521 \pm 0.119$ & $$ \\ 
15 & 75 & 2 & $0.418 \pm 0.088$ & $0.433 \pm 0.047$ & $0.547 \pm 0.119$ & $0.603 \pm 0.217$ & $$ \\ 
15 & 75 & 3 & $0.411 \pm 0.081$ & $0.439 \pm 0.054$ & $0.551 \pm 0.138$ & $0.651 \pm 0.229$ & $$ \\ 
15 & 75 & 4 & $0.423 \pm 0.105$ & $0.451 \pm 0.062$ & $0.557 \pm 0.109$ & $0.645 \pm 0.220$ & $$ \\ 
15 & 100 & 2 & $0.430 \pm 0.113$ & $0.408 \pm 0.040$ & $0.507 \pm 0.120$ & $0.501 \pm 0.058$ & $$ \\ 
15 & 100 & 3 & $0.429 \pm 0.133$ & $0.414 \pm 0.058$ & $0.494 \pm 0.120$ & $0.514 \pm 0.089$ & $$ \\ 
15 & 100 & 4 & $0.420 \pm 0.085$ & $0.412 \pm 0.058$ & $0.514 \pm 0.110$ & $0.511 \pm 0.078$ & $$ \\ 
\end{tabular}
\end{table}

}

\end{document}